\documentclass[twoside]{article}

\usepackage[accepted]{aistats2026}

\usepackage[utf8]{inputenc} % allow utf-8 input
\usepackage[T1]{fontenc}    % use 8-bit T1 fonts
\usepackage{hyperref}       % hyperlinks
\usepackage{url}            % simple URL typesetting
\usepackage{booktabs}       % professional-quality tables
\usepackage{amsfonts}       % blackboard math symbols
\usepackage{nicefrac}       % compact symbols for 1/2, etc.
\usepackage{microtype}      % microtypography
\usepackage{xcolor}         % colors
\usepackage{comment}

\usepackage[round]{natbib}
\usepackage{mathtools} % amsmath with fixes and additions
\usepackage{booktabs} % commands to create good-looking tables
\usepackage{tikz} % nice language for creating drawings and diagrams
\usepackage{amsmath}
\usepackage{amssymb}
\usepackage{amsthm}
\usepackage{subcaption} 
\RequirePackage{algorithm}
\RequirePackage{algorithmic}
\newcommand{\RightComment}[1]{\unskip\hfill\(\triangleright\)~#1}
\usepackage{graphicx}
\usepackage{multirow}
\usepackage{wrapfig}
\usepackage{todonotes}
\usepackage{enumitem}
\theoremstyle{plain}
\newtheorem{theorem}{Theorem}[section]
\newtheorem{proposition}[theorem]{Proposition}
\newtheorem{lemma}[theorem]{Lemma}

\theoremstyle{definition}

\theoremstyle{remark}

\DeclareMathOperator{\RR}{\mathbb{R}}

\DeclareMathOperator{\E}{\mathbb{E}}

\let\subparagraph\paragraph
\usepackage[compact]{titlesec}

\titlespacing{\section}{0pt}{1ex}{1ex}
\titlespacing{\subsection}{0pt}{1ex}{0.5ex}
\titlespacing{\subsubsection}{0pt}{0.8ex}{0.5ex}

\DeclareMathSizes{9}{8.2}{5}{3}
\newcommand{\sectionfontsize}{\fontsize{13pt}{15pt}\selectfont}
\newcommand{\subsectionfontsize}{\fontsize{12pt}{14pt}\selectfont}
\titleformat{\section}
  {\sectionfontsize\bfseries}
  {\thesection}{1em}{}

\titleformat{\subsection}
  {\subsectionfontsize\bfseries}
  {\thesubsection}{1em}{}

\begin{document}
\pagestyle{fancy}
\fancyhf{}              % clear header/footer
\fancyfoot[c]{\rm\thepage}

\runningauthor{Samaddar, Sun, Nilsson, Madireddy}
\twocolumn[

\aistatstitle{Efficient Flow Matching Using Latent Variables}

\aistatsauthor{ Anirban Samaddar \And Yixuan Sun}

\aistatsaddress{ Argonne National Laboratory \And  Argonne National Laboratory}

\aistatsauthor{Viktor Nilsson \And  Sandeep Madireddy}

\aistatsaddress{KTH Royal Institute of Technology \And  Argonne National Laboratory}
]

\begin{abstract}
Flow matching models have shown great potential in image generation tasks among probabilistic generative models. However, most flow matching models in the literature do not explicitly utilize the underlying clustering structure in the target data when learning the flow from a simple source distribution like the standard Gaussian. This leads to inefficient learning, especially for many high-dimensional real-world datasets, which often reside in a low-dimensional manifold. 
To this end, we present \texttt{Latent-CFM}, which provides efficient training strategies by conditioning on the features extracted from data using pretrained deep latent variable models. Through experiments on synthetic data from multi-modal distributions and widely used image benchmark datasets, we show that \texttt{Latent-CFM} exhibits improved generation quality with significantly less training and computation than state-of-the-art flow matching models by adopting pretrained lightweight latent variable models. Beyond natural images, we consider generative modeling of spatial fields stemming from physical processes. Using a 2d Darcy flow dataset, we demonstrate that our approach generates more physically accurate samples than competing approaches. In addition, through latent space analysis, we demonstrate that our approach can be used for conditional image generation conditioned on latent features, which adds interpretability to the generation process.
\end{abstract}
% Existing strategies for incorporating manifolds, which include data with an underlying multi-modal distribution, often require expensive training and frequently lead to suboptimal performance. 

\section{Introduction}
% \SM{do we characterize efficiency quantitatively and compare with other models? we should mention in the abstract/intro what we mean by efficiency}
% \AS{Try to learn/model/interpret the manifold of the data. Information in source distribution or introduce it in the flow. Either motivate computationally or by utility of generating the samples based on the similarity. Enables something between unconditional and conditional sampling.}

Flow Matching (FM) is a generative modeling framework that directly learns a vector field that smoothly transports a simple source distribution to a target data distribution \cite{lipman2023flow}.
% \AS{Maybe briefly describe diffusion and NFs and say FM generalizes both} Unlike diffusion models, which rely on iterative noising and denoising, or normalizing flows, which impose invertibility constraints, FM constructs flexible probability paths without requiring simulation-based training \cite{lipman2024guide}.
Compared to widely adopted diffusion generative models \cite{song2020score, ho2020denoising}, FM provides a framework for building deterministic
% \textbf{deterministic} 
transport paths for mapping a simple noise distribution into a data distribution.
In fact, the FM framework can also map between two arbitrary distributions \cite{liu2022flow, albergo2023stochastic}.
The FM transport paths may be constructed to promote beneficial properties such as shortness and straightness of paths, providing significant computational benefits \cite{tong2024improving}.
%in which the user can construct shorter and straighter flows from simple source distributions to complex data distributions. 
% \YS{it's not clear how it generalizes the paths. Is it correct to say FMs generalize the path of DMs to arbirary markov process beyond gaussian?}
%These transport paths can enable significant computational advantages and therefore efficient sampling and scalable training \cite{tong2024improving} compared to diffusion and normalizing flows. 
% compared to the diffusion models \cite{tong2024improving}. 
FM has been shown to be applicable beyond Euclidean spaces, and extensible to incorporate optimal transport principles and novel conditioning mechanisms to improve sample quality and expressivity \cite{tong2020trajectorynet, albergo2023stochastic, tong2024improving}. These developments have facilitated applications in various domains, including foundation models for video generation \cite{polyak2025moviegencastmedia}, molecular modeling \cite{hassan2024etflowequivariantflowmatchingmolecular}, and discrete data generation \cite{gat2024discreteflowmatching}. %Furthermore, connections between FM and score-based generative models \cite{lipman2023flow, tong2024improving} suggest its potential to unify deterministic and stochastic modeling paradigms, further broadening its applicability.

%\VN{Is this motivation relevant?}
Despite these strengths, a limitation of current FM models stems from the fact that the prior knowledge about the underlying structures of the target data (such as low-dimensional clusters) is not explicitly included in the modeling, which can potentially lead to inefficiency in the FM training and convergence. 
One component of FM that could incorporate such information is the source distribution. 
Typically, an isotropic Gaussian \cite{lipman2023flow, tong2024improving} is used as the source, in an independent coupling with the target. Only recently, FlowLLM \cite{sriram2024flowllm} modified the source distribution as the Large Language Model generated response that FM subsequently refines to learn a transport to the target distribution. However, a data-driven source for high-dimensional image datasets is challenging to learn since transporting a custom source distribution to the target requires specialized loss functions and sampling processes \cite{daras2022soft, wang2023patch}.
On the other hand, very few works have explored ways to incorporate the underlying clustering structure of the data \cite{jia2024structured, guo2025variational} during training. These works condition the flow from source to target distribution using latent variables. 
%The conditioning approach can simplify the training problem, where one can adopt the standard training/sampling strategies in diffusion and flow matching models with the data structure learned through latent variables informing the flow. 
However, these works often lead to suboptimal performance for high-dimensional datasets while requiring customized training strategies that are dataset-dependent and expensive, thus restricting their broader applicability.

To address these limitations, we adapt advances in deep latent variable modeling \cite{kingma2022autoencodingvariationalbayes} to FM, and propose a simple and efficient training and inference framework to incorporate data structure in the generation process. Our contributions are as follows --

\begin{itemize}[noitemsep,topsep=0pt,leftmargin=*]
    % \item We propose \texttt{Latent-CFM} framework which is a simplified training algorithm that trains/finetunes lightweight deep latent variable models for efficiently incorporating multi-modal data structures in the conditional flow matching models.

    \item We propose \texttt{Latent-CFM}, a FM training framework that efficiently incorporates and finetunes lightweight deep latent-variable models, enabling conditional FM networks to capture and leverage low-dimensional clustering structures of the data efficiently.

    \item We demonstrate the effectiveness of the proposed framework in significantly improving generation quality 
    % (FID $\sim 3.5$ on CIFAR10)
    % \SM{anything quantitative to say on the quality?}
    and training efficiency (up to $50\%$ fewer steps) compared to popular flow matching approaches in 2d synthetic mixture and popular image benchmark datasets such as MNIST, CIFAR10, and ImageNet. %\SM{which datasets}.

    \item We show the superiority of \texttt{Latent-CFM} in generating physically consistent data with experiments on the 2d Darcy flow dataset, where consistency is measured through the residual of the governing partial differential equation, in contrast to the natural image datasets. 
    %\AS{What additional results are we demonstrating? Adding interpretability? Add a 2d demonstration of the latent space?} \SM{need to revisit this. Do we explicitly apply the constraint?, so why is Latent-CFM better?}\YS{can we say the learned latent variables preserve the physical consistency between target fields? Although, this needs additional analysis to show.}\SM{lets decide this 5/15}

    \item We explore the ability of our method in feature-conditional generation and extend it to compositional generation from the product of two feature-conditioned distributions. This allows generation-based low-frequency features learned from compressed latent representations, allowing regeneration of images with similar characteristics; see Fig.~\ref{fig:cifar_style}.
    %\VN{i.e. the CIFAR10\_car\_blue figure.}
\end{itemize}

\section{Related works}
%\AS{probabilistic generative models for images - diffusion and FM}
%\YS{I think we can leave the broader FM and DM in the introduction and mention VAE, ICFM, OTFM, and VRFM here?}
Continuous flow-based generative models, like diffusion and flow matching, evolve samples from a (typically) simple source distribution to a complex target distribution \cite{song2020score,lipman2023flow, tong2024improving}. The most common choice for source distribution is Gaussian white noise. %This choice often leads to longer transport paths, resulting in inefficiency during the sampling process \cite{tong2024improving}.
%\AS{Effect of prior distributions}
%Few studies \cite{daras2022soft, wang2023patch} have attempted to solve this problem for diffusion models by directly modifying the source to have properties of the data distribution by injecting custom noise, e.g., blurring or masking. These works presents specialized training and sampling procedures needed to remove the effects of these noises. In flow matching, 
Many works have considered conditional generation based on these models \cite{dhariwal2021diffusion, ho2022classifierfreediffusionguidance, zheng2023guidedflowsgenerativemodeling}, e.g. based on class labels, or text captions (prompts), by incorporating this information in the input of the network.
This enables techniques such as classifier-free guidance, and typically leads to better overall performance.

\paragraph{Source sample structure} A different approach to generating samples based on structural information is to incorporate this in the source distribution. 
\cite{kollovieh2024flowmatchinggaussianprocess} attempts to model the source distribution as a Gaussian process (GP) to model the generation of time series data. In \cite{sriram2024flowllm}, the authors use a fine-tuned LLM to generate the source (noisy) samples to be transported by a discrete flow matching model for materials discovery.
\begin{figure*}
\centering
\hspace{-1cm}
\begin{subfigure}[b]{0.45\textwidth} % Width specified here
\centering
\begin{subfigure}[b]{0.33\textwidth} % Width specified here
    \centering
\includegraphics[width=\textwidth]{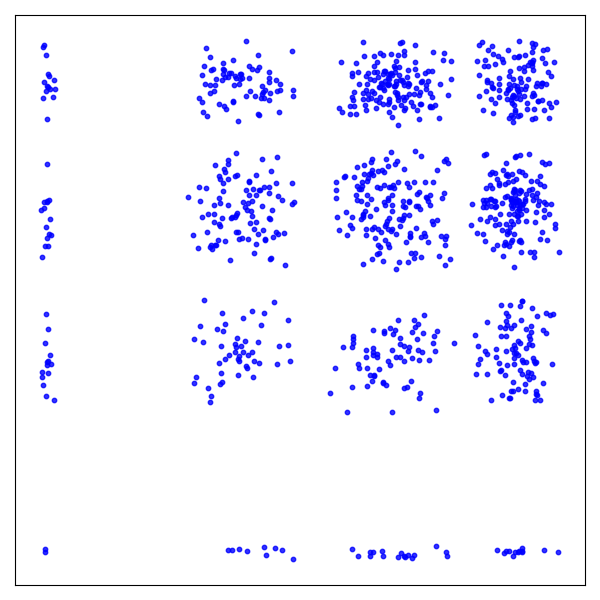}
    \caption{Data}
    \label{fig:test_sample1} % Added label for subfigure
\end{subfigure}
% \hspace{-0.2cm}
\begin{subfigure}[b]{0.33\textwidth} % Width specified here
    \centering
\includegraphics[width=\textwidth]{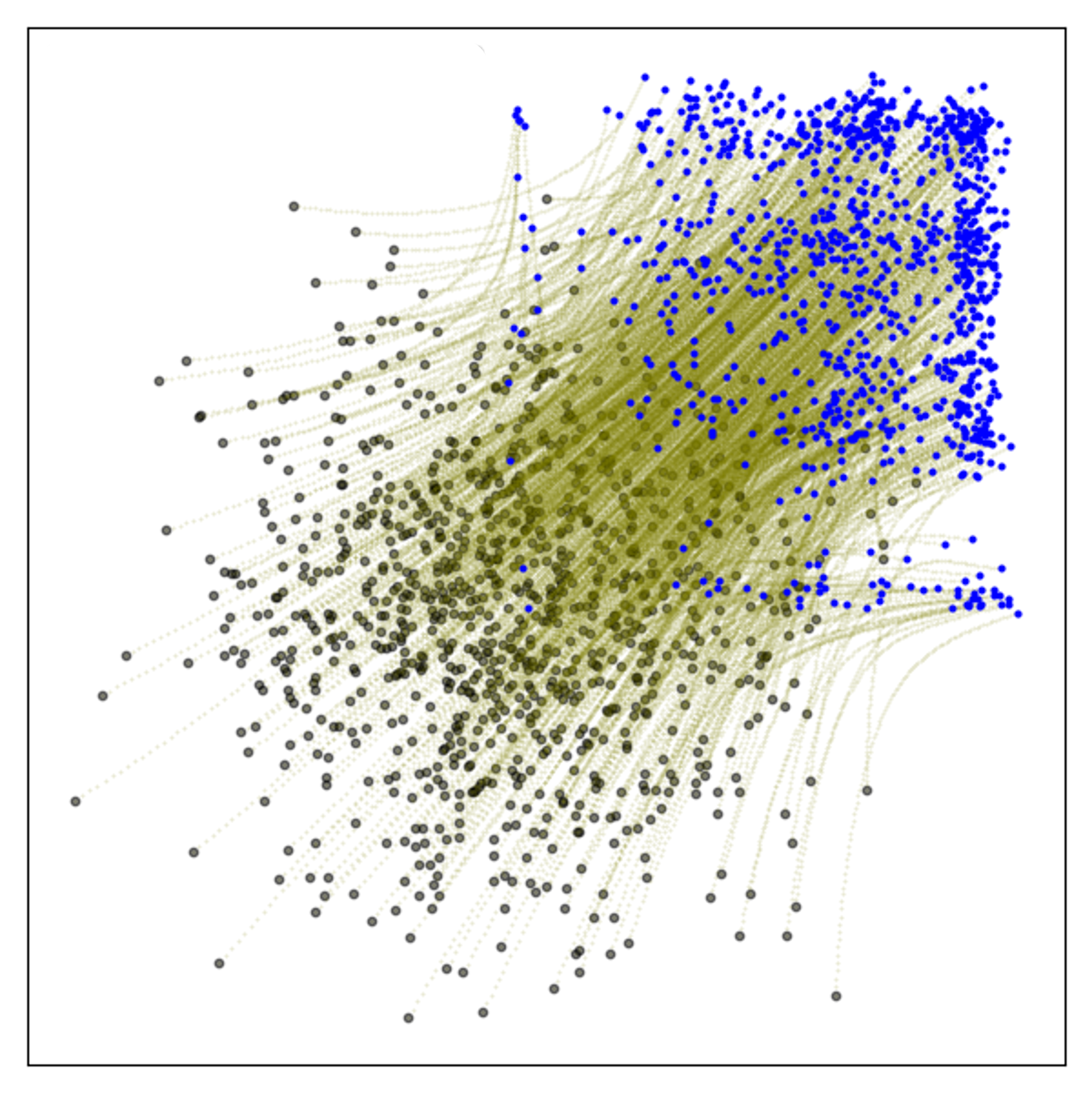}
    \caption{I-CFM}
    \label{fig:synthetic_ICFM1} % Added label for subfigure
\end{subfigure}
% \hspace{0.3cm}
\hfil
\begin{subfigure}[b]{0.33\textwidth} % Width specified here
    \centering
\includegraphics[width=\textwidth]{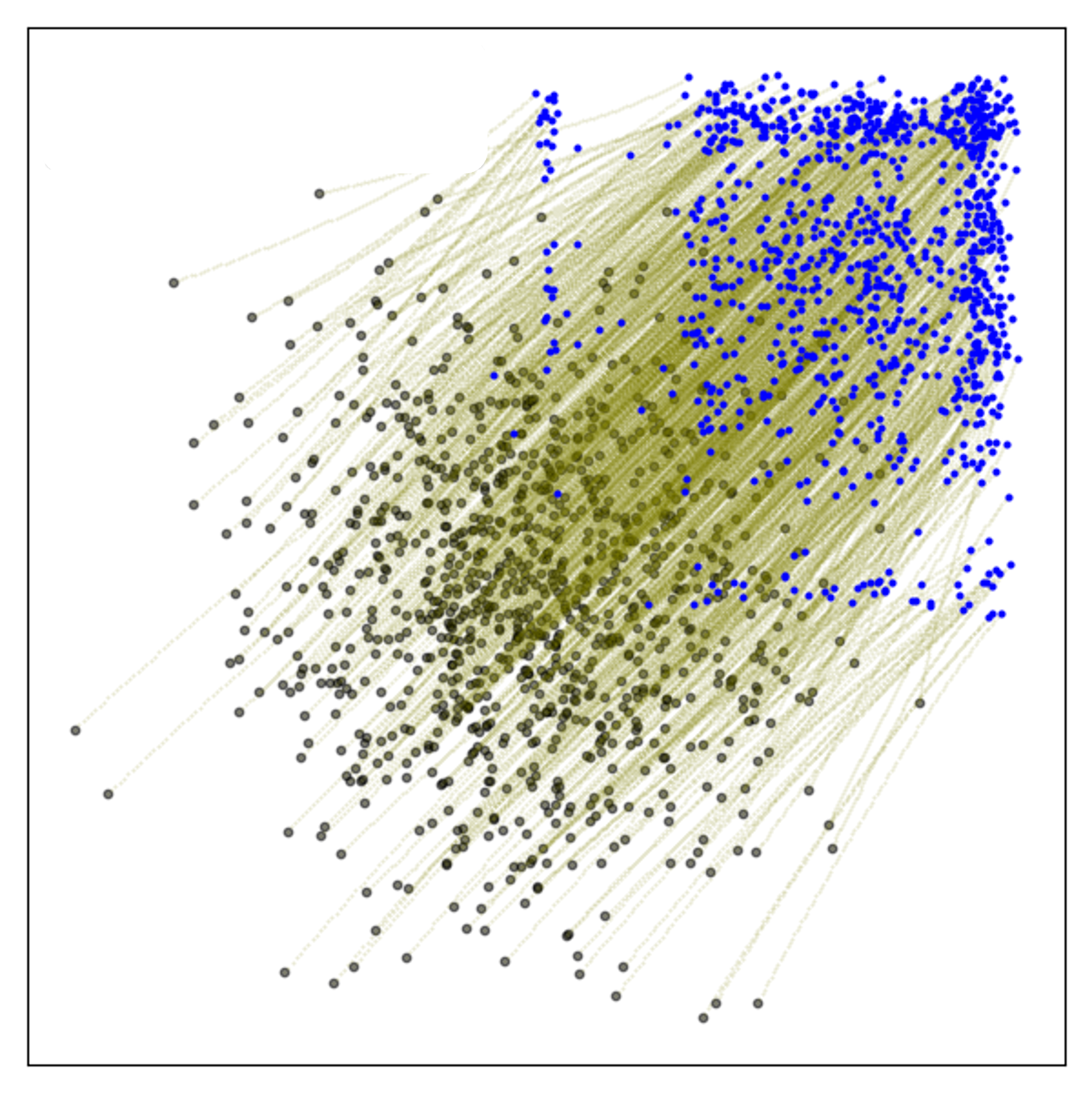}
    \caption{OT-CFM}
    \label{fig:synthetic_gaussianOT} % Added label for subfigure
\end{subfigure}
\begin{subfigure}[b]{0.33\textwidth} % Width specified here
    \centering
\includegraphics[width=\textwidth]{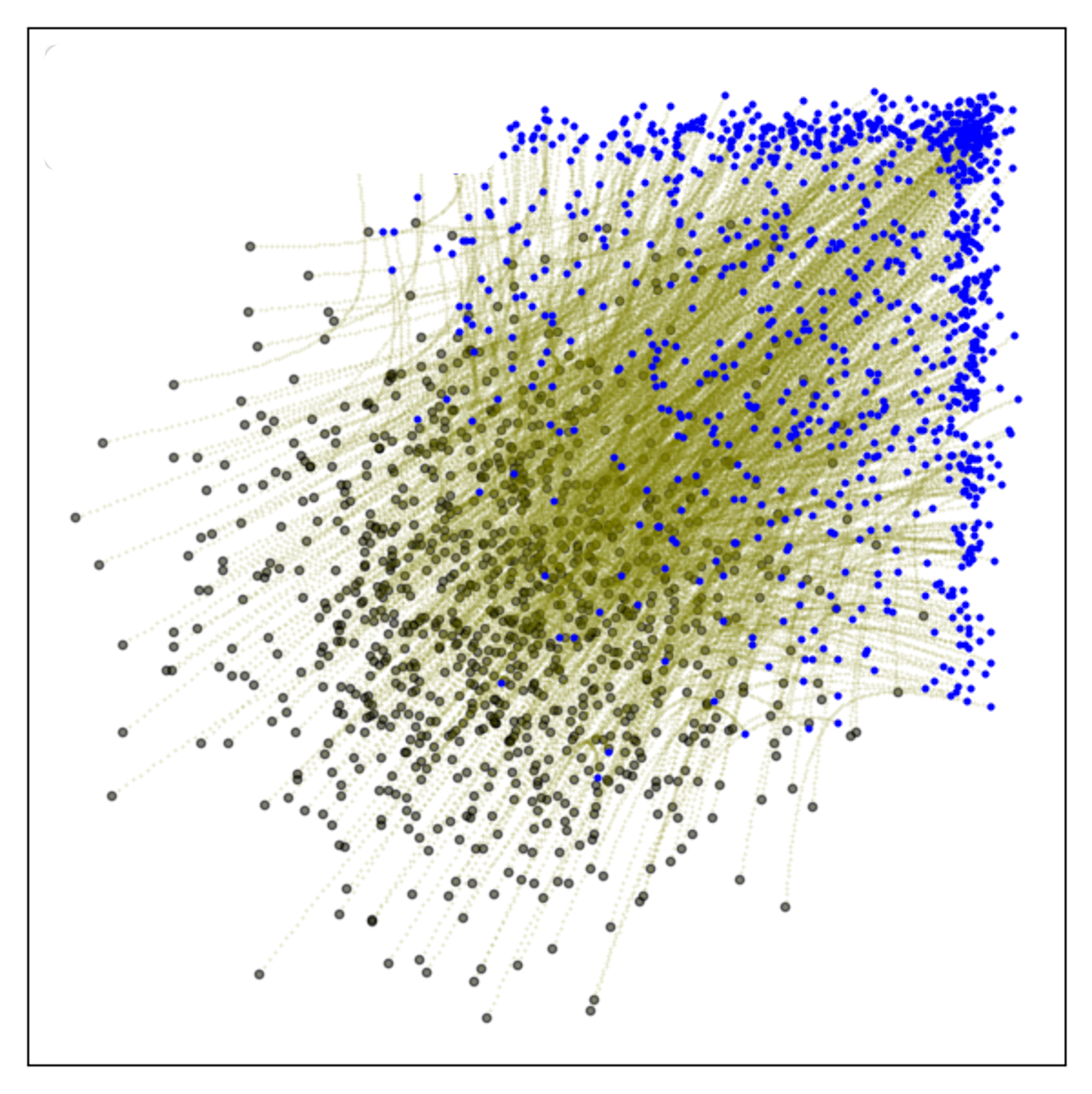}
    \caption{VRFM}
    \label{fig:synthetic_VRFM} % Added label for subfigure
\end{subfigure}
\end{subfigure}
% \hspace{-1cm}
\begin{subfigure}[b]{0.55\textwidth} % Width specified here
    \centering
\includegraphics[width=0.8\textwidth]{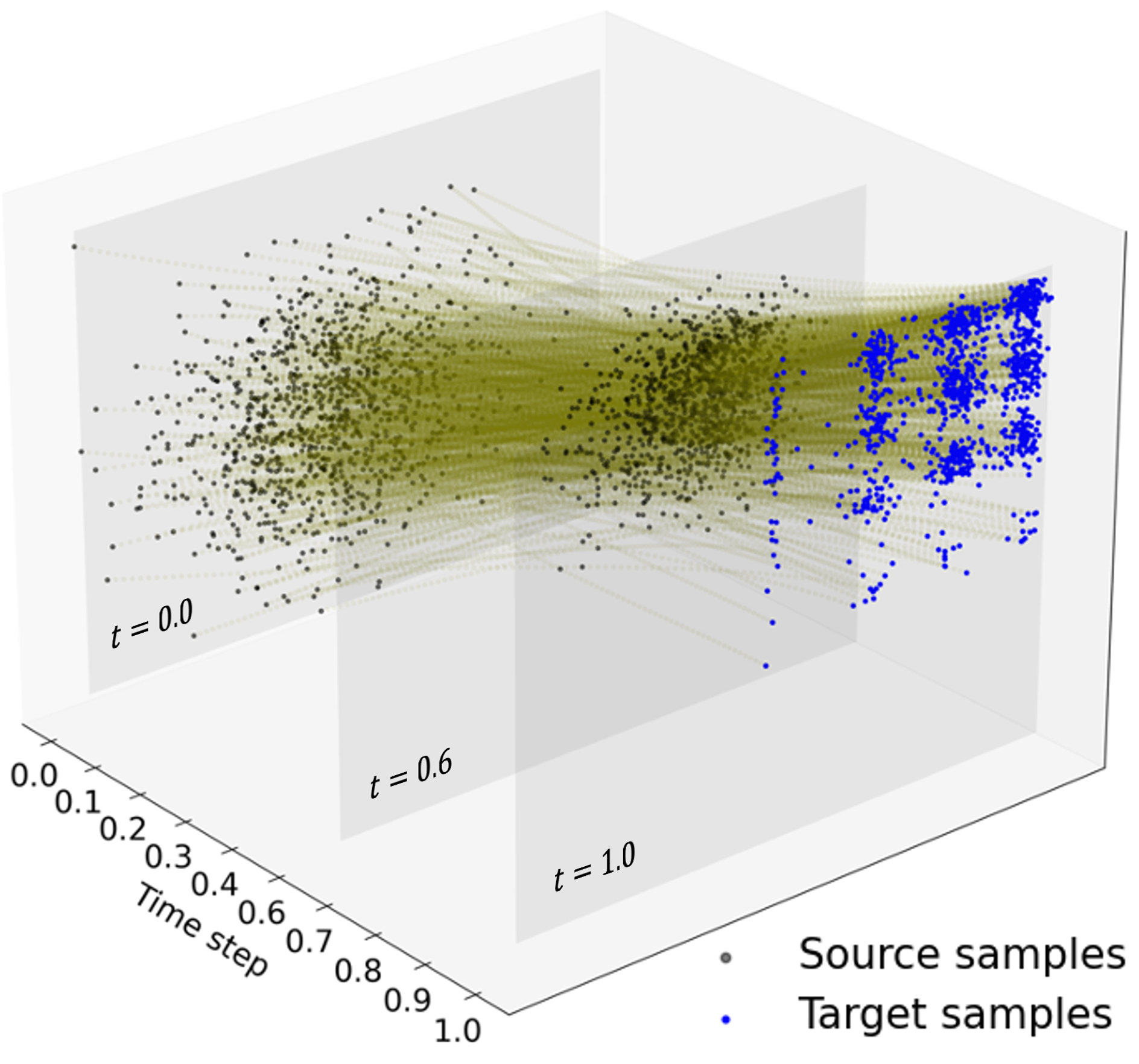}
    \caption{\texttt{Latent-CFM} (VAE)} 
    % \SM{if you have time, rotate this counter clockwise in the X-Y plane by 20-30 deg - just to see if the flow paths in 3D are more clear; if not keep this}}
    \label{fig:synthetic_2d_lcfm_vae} % Added label for subfigure
\end{subfigure}
% \hspace{-0.3cm}
    \caption{(a)-(d) showing generation quality of different flow-based generative models on 2d triangle dataset with 16 modes. Popular conditional flow matching approaches generated samples fail to capture the multi-modal structure of the target data distribution. (e) \texttt{Latent-CFM} generates samples similar to the data samples, capturing the multi-modal data structure.
    }
    \label{fig:synthetic_2d_cfm}
\end{figure*}
%\AS{Remove if necessary}
Diffusion Sch\"odinger bridges \cite{de2021diffusion} is an unsupervised framework for learning couplings $\pi$ between the source and target distributions such that target samples are close to their source, in accordance with some reference corruption process, the low-noise limit being optimal transport (OT) couplings.
%Another direction of improving efficiency of %incorporating data structure 
%% \SM{not clear what you mean by data structure here} \AS{Say about how the OT or SB ``aligns'' the source to the target} 
%the flow-based models is by constructing specialized flows \cite{tong2020trajectorynet, tong2024improving, de2021diffusion}. 
Another line of work in this direction uses discrete OT-couplings computed on batches for training the flow network \cite{pooladian_multisample_2023, tong2024improving}.
This similarly promotes OT-like proximity between source and target samples.

%These works focus on solving for the optimal transport path from the source to the target data distribution, which ``aligns'' the random draws from the source distribution to the target samples during training to ensure straighter and more efficient flows. Popular approaches in this direction include optimal transport conditional flow matching (OT-CFM) \cite{tong2024improving} and Schrödinger bridge conditional flow matching (SB-CFM). 
% These methods often produce straighter and more efficient flows. 
% However, in Fig.~\ref{fig:synthetic_gaussianOT} we show that the OT-CFM fails to capture the multi-modal data structure of the 2d triangle dataset with 16 modes (details in Appendix [REF]) and generates samples similar to the independent conditional flow matching models (I-CFM) (Fig.~\ref{fig:synthetic_ICFM1}).

\paragraph{Latent variables}
Recent studies \cite{guo2025variational, jia2024structured} in incorporating data structures in diffusion and flow matching models have focused on a modeling approach where the flow network is conditioned on a latent variable. In diffusion models, \cite{jia2024structured} proposed learning a Gaussian mixture model from data and conditioning the denoiser neural network with the learned cluster centers during training. The method shows promising results for 1-dimensional synthetic datasets but suboptimal results for high-dimensional datasets. In flow matching, \cite{guo2025variational} proposes to adapt deep latent variable models \cite{kingma2022autoencodingvariationalbayes} to cluster the conditional transport paths. The authors show promising results on high-dimensional datasets; however, the method demands expensive training. In Fig.~\ref{fig:synthetic_2d_cfm}, we show that the popular CFM models fail to capture the multi-modal data structure of the 2d triangle dataset with 16 modes (details in supplementary material).

% \SM{Need to say how our approach is different from the rest and what it enables} 
In this study, we present \texttt{Latent-CFM}, a framework for incorporating low-dimensional structures of the target data into the training/inference process of conditional flow matching models. Our approach enables adapting the popular deep latent variable models \cite{kingma2022autoencodingvariationalbayes} for efficient training and high-quality sample generation.  In addition, our approach can generate samples conditioned on data features, adding interpretability to the generated samples, which is uncommon for the standard flow matching approaches. Fig.~\ref{fig:synthetic_2d_lcfm_vae} shows that our approach can generate samples capturing the multi-modal structure of the data.  

\section{\texttt{Latent-CFM}: Conditional Flow Matching with Latent variables}
% \AS{Make sure to clarify the random variables and the associated distributions}
This section describes our proposed method \texttt{Latent-CFM}. First, we describe the notations that will be followed throughout the manuscript.
\subsection{Background and Notations}
 We denote the unknown density of the data distribution over $\mathbb{R}^d$ by $p_1(x)$ and the source density, which is known and easy to sample from, by $p_0(x)$. Generative modeling involves finding a map from the simple source density $p_0(x)$ to the complex data distribution $p_1(x)$. We denote $x_0,$ and $x_1$ as the random variables following the distributions $p_0(x)$ and $p_1(x)$ respectively.

% \subsection{Generalized Energy-based Model}
% A generalized energy-based model (GEBM) has two components: a base distribution and an energy-based model
% \cite{arbel_generalized_2021}. It is designed to address the issue where a standard EBM learns the density over the entire domain when the true data distribution is only supported by a lower dimensional manifold. On the other hand, while the generators in GAN learns the correct support, it fails to learn the correct probability mass over the support when there are disconnecting modes present. Therefore, the GEBM leverages the advantages of both framework and assign/learn density to the lower dimensional manifold, achieving accurate modeling of the true data distribution regarding both the support and probability mass.

% The GEBM has the following form
% \begin{equation}\label{eqn:gebm}
%     \mathbb{Q}(dx) = \exp(-E(x) - Z)\mathbb{G}(dx),
% \end{equation}
% where $\mathbb{Q}$ and $\mathbb{G}$ are the probability measure and base measure, repsectively, $E(x)$ is the energy function, and $Z = \log(\int \exp(-E(x))\mathbb{G}(dx))$.

\paragraph{Probability flow ODE:}

%Flow matching involves learning a vector field $u_t(.)$ which defines an ordinary differential equation:
A time-dependent vector field $u_t:[0,1]\times \mathbb{R}^d \rightarrow \mathbb{R}^d$ defines an ordinary differential equation, 
% \YS{should the initial condition of the ODE be \(x_0\)?}
\begin{equation}\label{eq:vec_field}
    \frac{d \phi_t(x)}{dt} = u_t(\phi_t(x)); \phi_0(x) = x_0
\end{equation}
where $\phi_t(x)$ is the solution of the ODE or \textit{flow} with the initial condition in Eq~\ref{eq:vec_field}, and $u_t(.)$ (interchangeable with $u(.,t)$) is the ground-truth vector field that transports the samples from the source to the target distribution. We denote $p_t(.)$ as the generated probability path by Eq.~\ref{eq:vec_field}, with $p_0(.), p_1(.)$ as the source and target distribution, respectively.
% \AS{can remove}Given a source distribution $p_0(x)$ one can learn an invertible flow $\phi_1(x)$ using maximum likelihood using the change of variable formula \cite{NEURIPS2018_69386f6b}.
% Eq.~\ref{eq:change_of_var}.
% This is the main motivation for learning normalizing flows (NF) where the source distribution is commonly assumed to be a standard Gaussian distribution.

\paragraph{Flow Matching:}
% Learning an invertible flow is restrictive in practice and thus serves as the motivation for flow matching. 
Flow matching involves learning the vector field $u_t(.)$ that generates the flow from source to target distribution using a neural network $v_\theta(.,t)$ by optimizing the loss:
\begin{equation}\label{eq:flow_matching}
    \mathcal{L}_{\text{FM}} = \E_{t, p_t(x)} \left[ ||v_\theta(x,t) - u_t(x)||_2^2 \right]
\end{equation}
% Given a learned vector field $v_t(.,\theta)$, it is possible to calculate the likelihood of the data using the instantaneous change of variable formula [CITE],
% \begin{align}\label{eq:instant_change_of_var}
%     \log p_1(x_1) = \log p_0(x_0)
% \end{align}
Given the marginal probability path $p_t(x) = N(x|\mu_t,\sigma^2_tI)$, 
% \YS{is the marginal path a Gaussian too?}
the ODE flow that generates the path is not unique. However, a simple choice of the flow is $\phi_t(\epsilon) = \mu_t + \sigma_t\epsilon; \epsilon\sim N(0,I)$. \cite[Theorem~3]{lipman2023flow} and \cite[Theorem~2.1]{tong2024improving} show that the unique vector field that generates the flow has the following form:
\begin{equation}\label{eq:normal_vec_field}
    u_t(x) = \frac{\sigma'_t}{\sigma_t}(x-\mu_t) + \mu'_t
\end{equation}
where, $\mu'_t = \frac{d\mu_t}{dt}; \sigma'_t = \frac{d\sigma_t}{dt}$.

\paragraph{Conditional Flow Matching:}
The probability path $p_t$ is unknown for general source and target distributions. \cite{lipman2023flow, tong2024improving} proposed conditional flow matching (CFM) where the probability path $p_t(.|x_0,x_1)$ and the vector field $u_t(.|x_0,x_1)$ is conditioned on the end-point samples $(x_0,x_1)$ drawn from the distribution $q(x_0,x_1)$. The marginal vector field $u_t$ and the probability path $p_t$ are given by:
\begin{equation}\label{eq:marginal_flow}
    p_t(x) = \int p_t(x|x_0,x_1) q(x_0,x_1) dx_0 dx_1
\end{equation}
\begin{equation}\label{eq:marginal_vfield}  
    u_t(x) = \int u_t(x|x_0,x_1) \frac{p_t(x|x_0,x_1)q(x_0,x_1)}{p_t(x)} dx_0 dx_1 
\end{equation}
Eq.~\ref{eq:marginal_flow} induces a mixture model on the marginal probability path $p_t(x)$ with the conditional probability paths $p_t(x|x_0,x_1)$ weighted according to the likelihood $q(x_0,x_1)$. Similarly the marginal vector field $u_t(x)$ is an weighted average of the conditional vector field $u_t(x|x_0,x_1)$ with the posterior likelihood $p_t(x_0,x_1|x) = \frac{p_t(x|x_0,x_1)q(x_0,x_1)}{p_t(x)}$ as weights.

Given that we know the conditional vector field $u_t(x|x_0,x_1)$, it is not possible to derive the marginal $u_t(x)$ since the denominator of the posterior $p_t(x_0,x_1|x)$ involves the intractable marginal probability path $p_t(x)$. Therefore,  \cite{lipman2023flow, tong2024improving} proposed conditional flow matching objective:
\begin{equation}\label{eq:cond_flow}
    \mathcal{L}_{\text{CFM}} = \E_{t,q(x_0,x_1),p_t(x|x_0,x_1)} \left[ ||v_\theta(x,t) - u_t(x|x_0,x_1)||_2^2 \right]
\end{equation}
The CFM objective requires samples from $p_t(x|x_0,x_1)$, and $q(x_0,x_1)$ and the target conditional vector field $u_t(x|x_0,x_1)$. 
% Remarkably 
\cite{lipman2023flow, tong2020trajectorynet} show that under mild conditions $\nabla_\theta \mathcal{L}_{\text{CFM}}(\theta) = \nabla_\theta \mathcal{L}_{\text{FM}}(\theta)$. Therefore, the learned vector field $v_{\theta^*}(x,t)$ that minimizes Eq.~\ref{eq:cond_flow} is also the minimizer of Eq.~\ref{eq:flow_matching}. As a consequence, we can directly solve Eq.~\ref{eq:vec_field} with replacing $u_t(.)$ by $v_{\theta^*}(.,t)$ to transport the source distribution samples to the target distribution.

% Eq.~\ref{eq:marginal_flow} is from the law of total probability. It is not obvious that integral in Eq.~\ref{eq:marginal_vfield} generates the marginal vector field $u_t$. Following Theorem 1 in [CITE Lipman et al.] and Theorem 3.1 in  [CITE Tong et al.], it can be shown that the marginal vector field $u_t(.)$ given by Eq.~\ref{eq:marginal_vfield} indeed generates the marginal probability path in Eq.~\ref{eq:marginal_flow}. Additionally, [CITE Lipman et al., Tong et al.] show that under mild conditions $\nabla_\theta \mathcal{L}_{\text{CFM}}(\theta) = \nabla_\theta \mathcal{L}_{\text{FM}}(\theta)$. Therefore, it is sufficient to train a neural network $v_t(.,\theta)$ to target the conditional vector field $u_t(x|z)$ by optimizing Eq.~\ref{eq:cond_flow}. The conditioning variable $z$ is commonly selected as a pair $z = (x_0,x_1)$ of source and data sample either selected independently (I-CFM) or jointly (OT-CFM). 

% \AS{Say about common conditional probability paths and generated vector fields}

% To train a CFM model, one needs to specify the conditioning density $q(z)$, the conditional probability path $p_t(x|z)$, and the target vector field $u_t(x|z)$. \AS{Add a table for different flows?} Popular choices involve (1) I-CFM, (2) OT-CFM, and (3) SB-CFM (CITE Tong et al.). %However, all of these method considers a simple prior density $p_0 = N(0,I)$.

%\section{\texttt{Latent-CFM}: Conditional Flow Matching with Latent variables}
\subsection{\texttt{Latent-CFM} framework}
A key ingredient in CFM training is to specify the conditioning distribution $q(x_0,x_1)$ to sample for the CFM training. There are several choices% , as highlighted in the TABLE
, such as independent coupling \cite{tong2024improving}, optimal transport \cite{tong2020trajectorynet}, etc. Motivated by the deep latent variable models (LVMs) such as VAEs \cite{kingma2022autoencodingvariationalbayes}, we propose \texttt{Latent-CFM} that combines recent developments of LVMs with flow matching for improved generative modeling. 
% \SM{add context on why you are discussing VRFM?}

% \subsection{\texttt{Latent-CFM} framework}

\begin{figure}
    \centering
    \includegraphics[width=.9\linewidth]{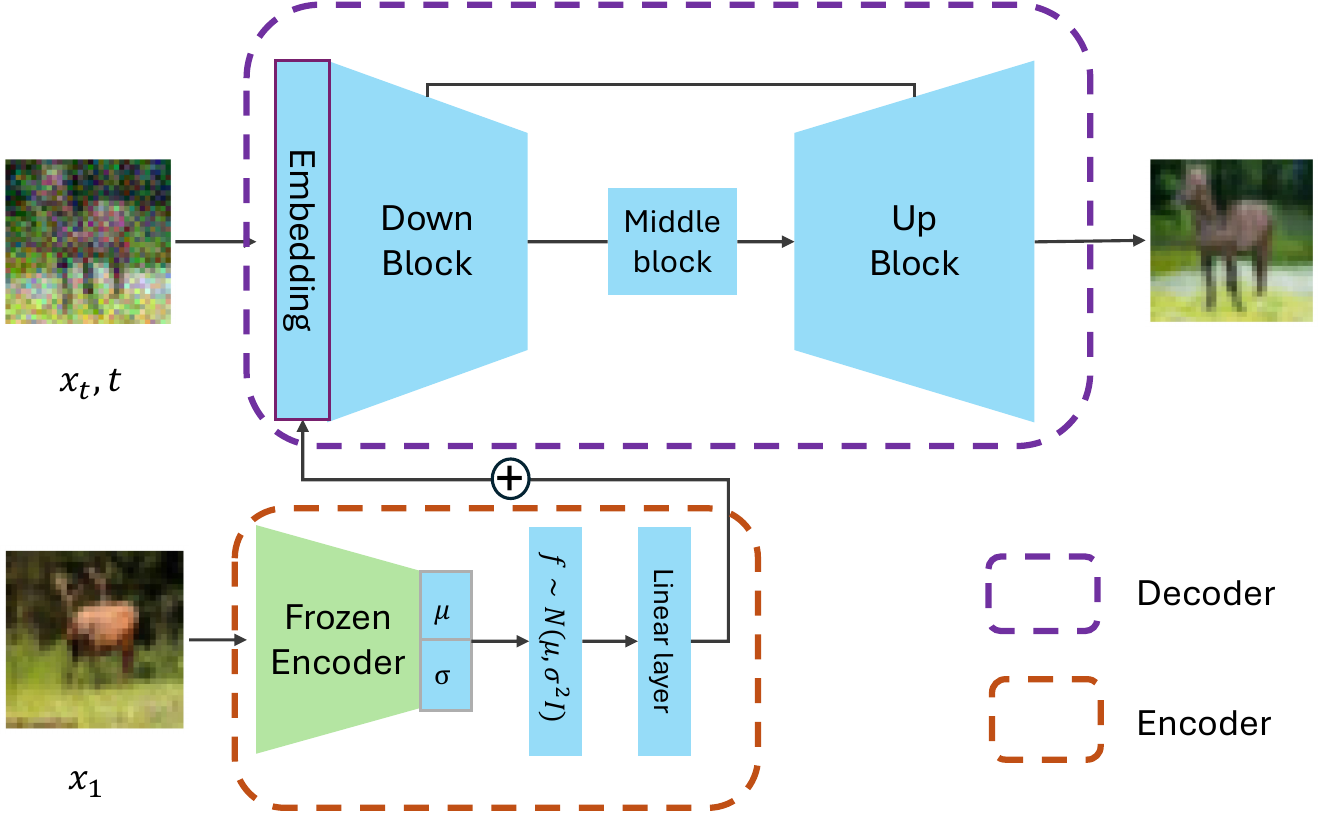}
    \caption{Schematic of \texttt{Latent-CFM} framework. Given a data $x_1$, \texttt{Latent-CFM} extracts latent features using a frozen encoder and a trainable stochastic layer. The features are embedded using a linear layer and added to the learned vector field. The framework resembles an encoder-decoder architecture like VAEs. }
    % \SM{Add a 2-3 sentence description. explaining the workflow}}
    % \AS{Move this beside Fig 1? Change the left picture to any random step of the denoising process}}
    \label{fig:schematic}
\end{figure}

% \AS{Present it in a more concise way}
We implicitly model $q(x_0,x_1)$ as a mixture distribution with mixture weights given by a latent variable $f$ with density $q(f)$ defined over $\RR^k$ where $k\leq d$,
\begin{equation}\label{eq:framework}
    q(x_0,x_1) = \int q (f) q(x_0,x_1|f) df
\end{equation}
$q(x_0,x_1|f)$ is the likelihood of the conditioning distribution given the variables $f$. In this study, we model $p_1$ as 
% \YS{conditionally?} 
independent of the source distribution $p_0$ where $p_1$ is assumed to be a mixture of conditional distributions conditioned on the latent random variable $f$,
\begin{equation}\label{eq:latent_cfm}
    q(x_0,x_1|f) = p_0(x_0) \times p_1(x_1|f)
\end{equation}

% Using the VAE [CITE] loss, we can write the lower-bound for the log-likelihood $q(z)$,
% \begin{align}\label{eq:VAE}
%     \E_{z\sim \mathcal{D}} \log q(z) &\geq \E_{z\sim \mathcal{D}} \big[ \E_{f \sim q_\phi (f|z)} \log q(z|f)  & \nonumber \\
%     &- D_{KL} (q_\phi(f|z)||p(f))\big]
% \end{align}
% where, $\mathcal{D} = p_0 \times p_{data}$ is the randomly drawn samples from the source and the training data, $q_\phi(.|z)$ is the posterior distribution given the data $z$, $p(f)$ is a user-specified prior distribution for $f$. 

% Note that, we cannot directly use the RHS of Eq.~\ref{eq:VAE} as the objective function since the likelihood $q(z|f)$ is unknown. In VAE, $z = x_1, \mathcal{D}=p_{data}$ and we model $q(x_1|f)$ using a decoder neural network. 
% % In \texttt{Latent-CFM}, we  

In Eq.~\ref{eq:latent_cfm}, the random variable $f$ represents latent variables that can be used to model the data distribution efficiently. 
% For discrete variables, 
% % Eq.~\ref{eq:latent_cfm} 
% $p_1$ can be written as a mixture model with $M$ components,
% \begin{equation}\label{eq:discr_latent_cfm}
%     p_1(x) = \sum_{j=1}^{M} w_j p_1(x|f=f_j)
% \end{equation}
% ,where $f = f_j, j=1,..,M$ are sample space of the variables with $\P(f=f_j) = w_j$. 
Eq.~\ref{eq:framework} aligns with the manifold hypothesis, which states that many real-world datasets tend to concentrate on low-dimensional manifolds. We aim to learn the latent variables from the data and augment them to aid the generative modeling.%  \AS{What is the advantage of the mixture model?}

% Given the distribution of the latent factor $q_{{\lambda}}$, we propose to model $p_1(x|{f})$ where ${f}\sim q_{{\lambda}}$ using conditional flow matching loss in Eq.~\ref{eq:cond_flow}. 
We introduce the \texttt{Latent-CFM} loss function,
\begin{align}\label{eq:latent_cond_flow}
    \mathcal{L}_{\texttt{Latent-CFM}} &= \E_{t,q(f),q(x_0,x_1|f),p_t(x|x_0,x_1)}  ||v_\theta (x,f,t) & \nonumber \\
    &- u_t(x|x_0,x_1)||_2^2,
\end{align}
where $q(x_0,x_1|f)$ is according to Eq.~\ref{eq:latent_cfm}. During sampling, we can generate samples from the latent distribution $f \sim q(f)$ and the source distribution $x_0 \sim p_0(x)$ and solve Eq.~\ref{eq:vec_field} replacing the vector field by $v_\theta (x,f,t)$ fixing $f$. Computing Eq.~\ref{eq:latent_cond_flow} requires us to sample from  $q(f)$, which is unobserved. Note that we can modify Eq.~\ref{eq:latent_cond_flow} into a more tractable objective,
\begin{align}\label{eq:latent_CFM_obj}
    \mathcal{L}_{\texttt{Latent-CFM}} &= \E_{t,q(x_0,x_1),q(f|x_0,x_1),p_t(x|x_0,x_1)}  ||v_\theta (x,f,t) & \nonumber \\ &- u_t(x|x_0,x_1)||_2^2 
\end{align}
We can apply Bayes theorem: $q(x_0,x_1)q(f|x_0,x_1) = q(f)q(x_0,x_1|f)$ to show the equivalence between Eq.~\ref{eq:latent_CFM_obj} and~\ref{eq:latent_cond_flow}. In Eq.~\ref{eq:latent_CFM_obj}, we can sample $(x_0,x_1) \sim q(x_0,x_1)$ and sample the posterior distribution $q(f|x_0,x_1)$ to compute the objective. Note that, the model in Eq.\ref{eq:latent_cfm} implies that $f$ is independent of the source $x_0$ and hence the posterior $q(f|x_0,x_1) = q(f|x_1)$. However, we keep the general notation $q(f|x_0,x_1)$, a more general model for the data where latent variables govern both source and target distributions. 
Proposition~\ref{prop:marginal-preserved} shows that if $v_\theta(x,f,t)$ has learned the minimum of $\mathcal{L}_{\texttt{Latent-CFM}}$, its flow generates the data distribution.

% DISCLAIMER: The Proposition prop1:opt_CFM does not hold, removing this part
% Proposition~\ref{prop1:opt_CFM} in appendix shows that $\mathcal{L}_{\texttt{Latent-CFM}}$ is an upper bound of $\mathcal{L}_{CFM}$ in Eq.~\ref{eq:cond_flow} with $v_\theta(x,t) = \E_{q(f|x_0,x_1)} v_\theta (x,f,t)$. 
% Therefore, in this study, we optimize Eq.~\ref{eq:latent_CFM_obj} which also reduces $\mathcal{L}_{CFM}$. Note that, the equality holds when $v_\theta (x,f,t) = \E_{q(f|x_0,x_1)} v_\theta (x,f,t)$ almost surely w.r.t $q(f|x_0,x_1)$. This implies when $v_\theta (x,f,t)$ is constant for varying variables $f \sim q(f|x_0,x_1)$ and fixed $x,t,\theta$, minimizing the \texttt{Latent-CFM} loss is equivalent to minimizing the CFM loss. 

\subsubsection{Choice of $q(.|x_1)$}
The latent posterior distribution $q(.|x_1)$ should be easy to sample from and capture high-level structures in the data. 
In this study, we set $q(.|x_1)$ to be the popular variational autoencoders (VAE) \cite{kingma2022autoencodingvariationalbayes} for their success in disentangled feature extraction in high-dimensional datasets. The details about the VAEs are presented in the Appendix~\ref{sec:VAE}. In addition, we also explore Gaussian mixture models (GMM) \cite{pichler2022differentialentropyestimatortraining} for their efficient training on small-dimensional generative modeling, which we describe in the supplementary material.

In this study, we pretrain the VAEs optimizing the standard negative ELBO loss in Eq.~\ref{eq:loss_VAE} on the datasets and use the pretrained encoder $q_{\hat{\lambda}}(f|x_1)$ as a feature extractor. When using the encoder in \texttt{Latent-CFM}, we finetune the final layer (parameterized by $\lambda_{final}$) that outputs $(\mu, \log(\sigma))$ and we regularize its learning with a KL-divergence term added to Eq.~\ref{eq:latent_CFM_obj}:
% \begin{align}\label{eq:latent_CFM_KL_obj}
%     \mathcal{L}_{\texttt{Latent-CFM}} &\geq \E_{q(x_0,x_1)} \big[\E_{t,q(f|x_0,x_1),p_t(x|x_0,x_1)}   ||v_\theta (x,f,t) \nonumber \\ 
%     &- u_t(x|x_0,x_1)||_2^2 & \nonumber \\ &+ \beta D_{KL}(q_{\lambda_{final}}(f|x_0,x_1) || p(f)) \big]
% \end{align}
\begin{align}\label{eq:latent_CFM_KL_obj}
    \mathcal{L}_{\texttt{Latent-CFM}} \leq & \E_{q(x_0,x_1)} \big[ \nonumber \\
    \E_{t,q(f|x_0,x_1),}&{}_{p_t(x|x_0,x_1)}  ||v_\theta (x,f,t) \nonumber 
    - u_t(x|x_0,x_1)||_2^2 & \nonumber \\ &+ \beta D_{KL}(q_{\lambda_{final}}(f|x_0,x_1) || p(f)) \big]
\end{align}
where, $p(f) = N(0,I)$. Eq.~\ref{eq:latent_CFM_KL_obj} is an upper bound of Eq.~\ref{eq:latent_CFM_obj} since KL-divergence is non-negative. 
% This implies that Proposition~\ref{prop1:opt_CFM} also holds true for Eq.~\ref{eq:latent_CFM_KL_obj}. 
In addition, we have empirically observed that the KL-divergence term resulted in the model learning beyond reconstruction of the data $x_1$ and increasing variability in unconditional generation. The loss in Eq.~\ref{eq:latent_CFM_KL_obj} is similar to VRFM loss \cite{guo2025variational} in Eq.~\ref{eq:mixture_vec}. However, in \texttt{Latent-CFM} the encoder $q_{\lambda_{final}}(.|x_0,x_1)$ only depends on the endpoints, which enables sample generation conditioned on the data features. Sec.~\ref{sec:vrfm} discusses the difference between the two loss functions in more detail.

Fig.~\ref{fig:schematic} shows the schematic of the \texttt{Latent-CFM} model. Given data $x_1$, it passes through the frozen encoder layer to output the latent variable $z$. The latent variables are then embedded through a linear layer and added to the neural network $v_\theta(.,.,z)$. The \texttt{Latent-CFM} model in Fig.~\ref{fig:schematic} can be viewed as an encoder-decoder (red and purple boxes) architecture where the encoder extracts the features and the decoder reconstructs the sample conditioned on the features. However, \texttt{Latent-CFM} learns to predict the vector field conditioned on the features from the encoder, which is integrated to reconstruct the final image.   

% \AS{Talk about what is similar and contrasting with the VRFM loss. We can also point to the encoder-decoder view of the latent-CFM and the benefit of conditional generation.} 
% The KL term in Eq.~\ref{eq:latent_CFM_KL_obj} regularizes learning information 

% Note that, optimizing Eq.~\ref{eq:latent_cond_flow} is different from learning a conditional generative model [CITE Guidance-free diffusion] where the target is to generate samples from 

% \paragraph{Reusability of existing theories in CFM}

% \paragraph{Generalization of CFM}

% \subsection{Modeling choices for Latent variables}

\subsubsection{Algorithm}
The training algorithm for \texttt{Latent-CFM} is described in Alg.~\ref{alg:vae-cond}. In this study, following I-CFM \cite{tong2024improving}, we adopt the conditional probability flow $p_t(x|x_0,x_1) = N(x|tx_1+(1-t)x_0,\sigma^2I)$, and the vector field $u_t(x|x_0,x_1) = x_1 - x_0$. We propose to pretrain a VAE model before the CFM training loop using the same training set. However, one can run a training loop where the VAE encoder and the vector field parameters are updated jointly at each step. We omit this training algorithm since it is similar to the VRFM \cite{guo2025variational} method with the key difference in choosing a static encoder $q_{\lambda}(.|x_0,x_1)$ whose parameters are trained along with the vector field. 
% In our experiments, separating the two steps has produced better results. We have observed that directly using the latents, without finetuning, from the pretrained VAE leads to suboptimal results. 
Finetuning the last layer also enables regularizing the information learned in the latent space through the KL term in \ref{eq:latent_CFM_KL_obj}.   

Alg.~\ref{alg:vae-cond_inference} describes the inference procedure. During inference, we need to draw samples from the estimated marginal distribution of the variables $\hat{p}(f) = \int p_1(x) q_{\hat{\lambda}}(f|x) dx$. For a moderately high-dimensional latent space, sampling from the marginal is difficult. Therefore, we reuse the empirical training samples $(x_1^{train},...,x_K^{train})$, for a given sample size $K$, to draw samples $f_i \sim q_{\hat{\lambda}}(f|x_i^{train})$ for all $i=1,...,K$.
\begin{figure}[t]
    \centering
    \includegraphics[width=\linewidth]{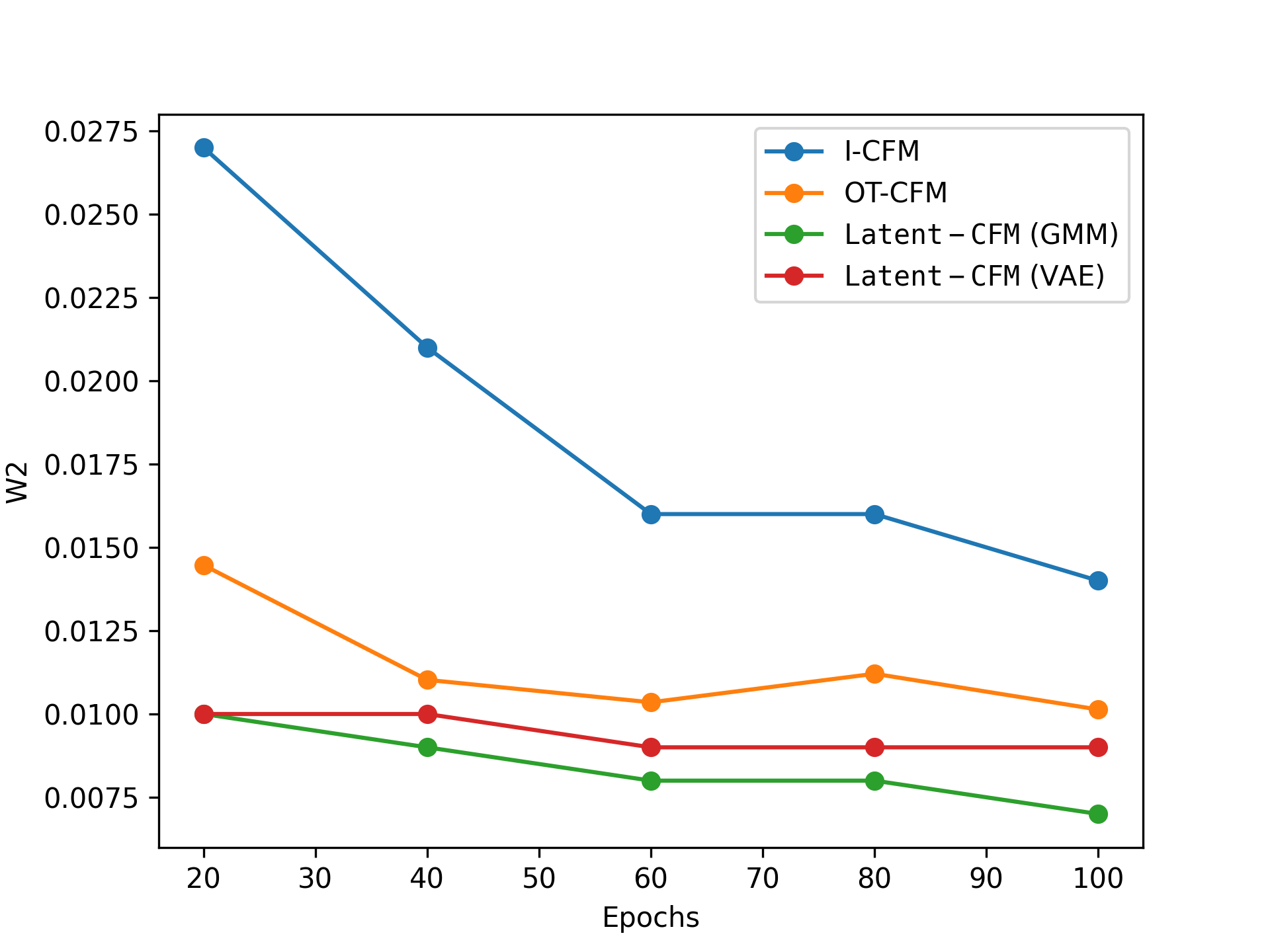}
    \caption{W2 vs epochs on 2d synthetic dataset}
    \label{fig:synthetic_eff}
\end{figure}

\section{Experiments}
%In this section, we compare the proposed \texttt{Latent-CFM} model with state-of-the-art generative modeling approaches on various tasks on synthetic and benchmark datasets.
We compare the proposed \texttt{Latent-CFM} against I-CFM~\cite{lipman2023flow}, OT-CFM~\cite{tong2024improving}, and VRFM~\cite{guo2025variational} on (a) unconditional data generation using synthetic 2d and high-dimensional image datasets and (b) physical data generation. We then analyze the learned latent space in the \texttt{Latent-CFM} model. Implementation details for all experiments are presented in the appendix.  
% \YS{made changes; check if to accept.}

% The proposed \texttt{Latent-CFM} model requires pretraining a VAE model to generate the latent features during training and inference. For a number of benchmark datasets

% We compare these models on generation of multi-modal 2d synthetic datasets,   

\subsection{Synthetic data sets}

We use the 2d Triangle dataset~\cite{pichler2022differentialentropyestimatortraining, nilssonremedi} to benchmark the models' generation quality.
%We consider the 2d Triangle dataset \cite{nilssonremedi}, which is commonly used to benchmark likelihood estimators. 
%Fig.~\ref{fig:test_sample} shows the generated samples from the 2d triangle dataset showing 16 modes with varying density of data samples. 
The data distribution contains 16 modes (Fig.~\ref{fig:synthetic_2d_cfm}(a)) with different densities, and we generate 100K samples and divide them equally between the training and testing datasets.
\begin{table}[t]
\centering 
\begin{tabular}{|c|c|}
   \hline
    \textbf{Method} & \textbf{W2} ($\downarrow$) \\
	\hline
    OT-CFM & $0.010 \pm 0.0031$ \\
    I-CFM & $0.014 \pm 0.0066$ \\
	VRFM & $0.050 \pm 0.0344$ \\
    \texttt{Latent-CFM} (VAE)  & $\mathbf{0.009 \pm 0.0013}$  \\
    \texttt{Latent-CFM} (GMM)  & $\mathbf{0.007 \pm 0.0008}$ \\
   \hline
   \end{tabular}
   \caption{Wasserstein-2 distance between the generated samples from the models and the test samples on 2d Triangle datasets. The mean and the standard deviations are calculated across 5 random data density shapes. \texttt{Latent-CFM} shows the most similarity with the test samples.}
    \label{tab:w2_synthetic}
\end{table}

All evaluated models share the same neural network architecture for the learned vector field. For VRFM, we fix the latent dimension to  2 and the encoder architecture to be similar to the vector field network. For the \texttt{Latent-CFM}, we consider two variants differing in their pre-trained feature extractors: (1) a 16-component Gaussian mixture model (GMM), and (2) a continuous VAE with 2d latent space with $\beta=0.1$. The training/inference algorithms for \texttt{Latent-CFM} with GMM are described in the appendix.
For sampling, we have used the \texttt{dopri5} solver to solve the ODE in Eq.~\ref{eq:vec_field}. 
\begin{figure*}
    \centering
    \begin{subfigure}[b]{0.41\textwidth}
        \centering
        \includegraphics[width=\linewidth]{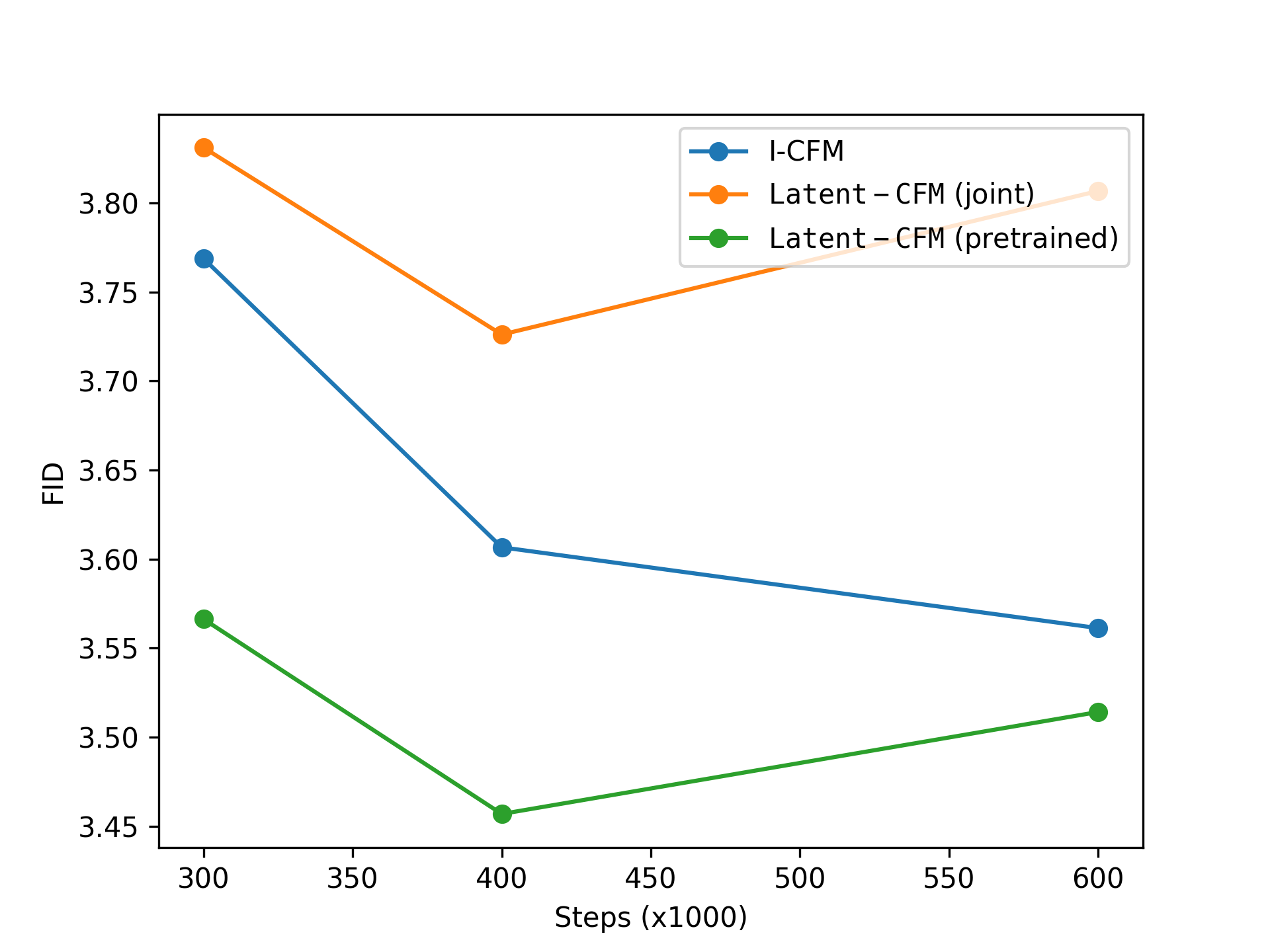}
        \caption{CIFAR10}
    \end{subfigure}
    \begin{subfigure}[b]{0.41\textwidth}
        \centering
        \includegraphics[width=\linewidth]{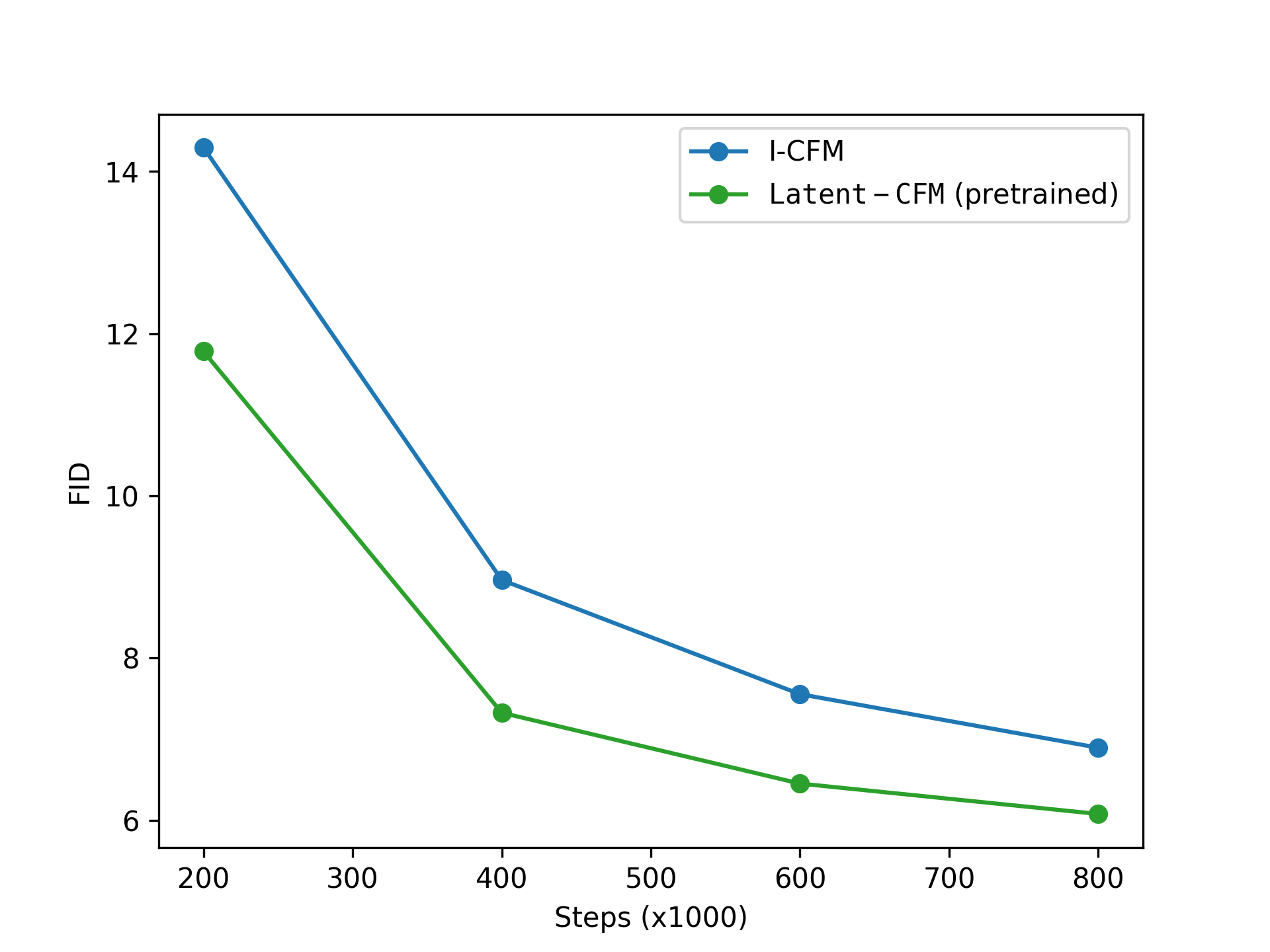}
        \caption{ImageNet}
    \end{subfigure}
    \caption{FID vs training steps show that \texttt{Latent-CFM} with pretrained VAE shows better generation quality compared to baseline methods on both CIFAR10 and ImageNet datasets early in the training process.}
    % \AS{Add more steps}}
    \label{fig:fid_steps}
\end{figure*}
\begin{table*}[t]
\centering
\resizebox{\textwidth}{!}{   
\begin{tabular}{|c|c|c|c|c|c|c|c|c|c|c|c|}
   \hline
   \multirow{3}{1.5cm}{\textbf{Methods}}  & \multicolumn{4}{c|}{\textbf{CIFAR10}}  &  \multicolumn{4}{c|}{\textbf{MNIST}} & \multicolumn{3}{c|}{\textbf{ImageNet}}\\
   \cline{2-5} \cline{6-9} \cline{10-12}& \multirow{2}{1.9cm}{\textbf{\# Params.}} & \multicolumn{3}{c|}{\textbf{FID ($\downarrow$)}}  & \multirow{2}{1.9cm}{\textbf{\# Params.}} & \multicolumn{3}{c|}{\textbf{FID ($\downarrow$)}} & \multirow{2}{1.9cm}{\textbf{\# Params.}} & \multicolumn{2}{c|}{\textbf{FID ($\downarrow$)}}\\
   \cline{3-5} \cline{7-9} \cline{11-12}
   & & \textbf{100} & \textbf{1000} & \textbf{Adaptive} & & \textbf{100} & \textbf{1000} & \textbf{Adaptive}& & \textbf{ODE} & \textbf{SDE}\\
	\hline
    % VRFM-1 \cite{guo2025variational} &37.2 M & 4.349& 3.582 & 3.561& - & - & - & - \\
    % VRFM-2 \cite{guo2025variational} &37.2 M& 4.484 & 3.614 & \textbf{3.478} & - & - & - & - \\
    % \hline
	OT-FM & 35.8 M & 4.661  & 3.862  & 3.727 & 1.56 M & 15.101  & 15.880  & 16.012  & - & - & - \\
    I-CFM & 35.8 M & 4.308  & \textbf{3.573}  & 3.561 & 1.56 M & 14.272  & 14.928  & 15.050 & 675.1M & 6.893 & 6.745\\
    \hline
    \texttt{Latent-CFM} (joint) & 47.3 M & 4.675  & 3.931  & 3.807  & 1.58 M & \textbf{13.818}  & 14.572  & \textbf{14.674} & - & - & - \\
    \texttt{Latent-CFM} (pretrained) & 36.1 M & \textbf{4.246}  & 3.575  & \textbf{3.514}  & 1.58 M & 13.848  & \textbf{14.543}  & 14.694 & 675.2M & \textbf{6.076} & \textbf{5.955}\\
    % \texttt{Latent-CFM} (large) & 34.5 M & 4.246 \semitransp{(0.0)} & 3.575 \semitransp{(0.0)} & 3.514 \semitransp{(0.0)} \\
   \hline
   \end{tabular}
   }
   \caption{Image generation performance of \texttt{Latent-CFM} compared to I-CFM and OT-FM on natural image datasets. Our method exhibits improved (or similar) FID over the state-of-the-art methods using both fixed-step Euler and the adaptive \texttt{dopri5} solver for all datasets.}
    \label{tab:uncond_generation}
\end{table*}

\paragraph{Training efficiency} Fig.~\ref{fig:synthetic_eff} shows the Wasserstein-2 (W2) distances with the test data vs the training epochs for the two \texttt{Latent-CFM} variants and the I-CFM and OT-CFM on the 2d synthetic dataset. We observe that \texttt{Latent-CFM} methods show lower W2 distances than the baseline methods for all training epochs. \texttt{Latent-CFM} with the GMM achieves a marginally better result than with the VAE encoder.   

\paragraph{Generation quality} Fig.~\ref{fig:synthetic_2d_cfm}(b)-(d) show the generation trajectories (yellow lines) of I-CFM, VRFM, and OT-CFM 
% \YS{Fig 3 currently only shows VRFM and latent CFM.} 
from the source to the target samples on the 2d triangle dataset. Fig.~\ref{fig:synthetic_2d_lcfm_vae} shows a 3d trajectory plot for \texttt{Latent-CFM} (VAE) with the samples from the time steps $[0,0.6,1]$ highlighted. Compared to the other models, \texttt{Latent-CFM} generates samples that present all modes of the true data distribution. 
% We observe that \texttt{Latent-CFM}, starting from white noise Gaussian samples ($t=0$), successfully incorporates multi-modality in the time steps between $t=0.6$ and $0.1$.  
% We compute the Wasserstein-2 (W2) metric between the generated and the test samples (details in the appendix) to quantify the generation quality. 
Table~\ref{tab:w2_synthetic} shows the summary (mean $\pm$ standard deviation) of W2 metrics of all models, where
% \YS{What's the reason of not comparing with OT-CFM?}. 
the \texttt{Latent-CFM} variants exhibit lower W2 distances from the test samples than the competing methods. The GMM variant of \texttt{Latent-CFM} achieves the lowest W2 score. 

The VRFM method shows the highest W2 distance among all methods. Fig.~\ref{fig:VRFM_comparison} shows the effect of simplifying the VRFM input from $(x_0,x_1,x_t,t)$ to $x_1$ for the encoder model (details in Sec.~\ref{sec:comp_vrfm}). The plot shows a significant improvement in the generation with an improved final W2 distance of $0.015$ vs the $0.050$ in Table~\ref{tab:w2_synthetic}. This demonstrates that it is easier to learn with the \texttt{Latent-CFM} encoder model, $q(.|x_1)$, with the data as input, due to the underlying multimodal data distribution. 
% \YS{is the center message of the paragraph that changing the input from the tuple to \(x_1\) essentially makes VRFM into LCFM? It might be a bit confusing to read. Would it be better to just point out the subtlety in one sentence and refer to the appendix for details? }

\subsection{Unconditional Image Generation}
For image generation, we train the OT-CFM, I-CFM, and \texttt{Latent-CFM} on MNIST, CIFAR10, and ImageNet $256 \times 256$ datasets. The details of the datasets are in the appendix.  
On MNIST and CIFAR10, we followed the network architectures and hyperparameters from~\cite{tong2024improving} to train OT-CFM and I-CFM. We did not find an open-source implementation for VRFM and implemented two VRFM variants with inputs (1) $(x_1,x_0,x_t,t)$, and (2) $(x_1,t)$ on CIFAR10. 
% Therefore, on CIFAR10, we add the FID of the top two VRFM models from \cite{guo2025variational} Table 1. 
% \YS{use present tense to be consistent?} 
We train two variants of \texttt{Latent-CFM} with: %(1) jointly trained, and (2) pretrained encoder model. 
(1) a \textbf{pretrained} encoder where only $\lambda$ is updated during training of the vector field network, and (2) an encoder \textbf{jointly} whose full parameter set is trained from scratch together with the vector field network.
%Note that \texttt{Latent-CFM} with a jointly trained encoder is similar to the VRFM model, with the key difference in the input to the encoder model $q(.|x_1)$. \texttt{Latent-CFM} shares the same neural network architecture for the vector field as the other models. In addition, we use open-sourced pretrained VAE on MNIST\footnote{https://github.com/csinva/gan-vae-pretrained-pytorch} and CIFAR10\footnote{https://github.com/Lightning-Universe/lightning-bolts/} as feature extractors as described in Alg.~\ref{alg:vae-cond}. We fix the latent dimension to be $20$ for MNIST and $256$ for CIFAR10. 
%The KL regularizer $\beta$ controls the variation in the generated sample as we vary $f$. 
We fix $\beta = 0.005$ for MNIST and $\beta = 0.001$ for CIFAR10 in Eq.~\ref{eq:latent_CFM_KL_obj}. 
% The parameter $\beta$ was selected to prevent \texttt{Latent-CFM} from only reconstructing the training set during sampling with Alg.~\ref{alg:vae-cond_inference}.  
The analysis on the impact of $\beta$ is in Sec.~\ref{sec:sel_beta}. We train all models for 600K steps on CIFAR10 and 100K on MNIST. For both datasets, we use both a fixed-step Euler solver for 100 and 1000 steps and an adaptive \texttt{dopri5} solver for solving the ODE. 
% \YS{sampler for the two datasets not clear}

On ImageNet, we compare \texttt{Latent-CFM} with the SiT model \cite{ma2024sitexploringflowdiffusionbased}, which is an I-CFM model. We train both models for 800K steps. We pretrain a VAE for 200K steps and use the encoder for training \texttt{Latent-CFM}. We fix $\beta = 0.001$ and the latent dimension to be 128 for the \texttt{Latent-CFM} training.  We follow the Stable Diffusion \cite{rombach2022highresolutionimagesynthesislatent} architecture for training flows embedded in a latent space of smaller dimension. We were unable to find an open-source implementation for VRFM on ImageNet. We use an adaptive \texttt{dopri5} solver for ODE sampling and an Euler-Maruyama sampler for 250 steps for SDE sampling.

% \begin{figure}
%     \centering
%     \includegraphics[width=\linewidth]{Figures/fid_vs_steps_v1.png}
%     \caption{FID vs training steps on CIFAR10 shows that \texttt{Latent-CFM} shows better generation quality compared to I-CFM early in the training process.}
%     % \AS{Add more steps}}
%     \label{fig:fid_steps}
% \end{figure}

\paragraph{Training efficiency} Fig.~\ref{fig:fid_steps} shows the FID vs training steps for I-CFM and \texttt{Latent-CFM}
 % the two best models 
on CIFAR10 and ImageNet. On both datasets, \texttt{Latent-CFM} with pretrained VAE exhibits significantly lower FID than I-CFM across training steps, demonstrating efficiency. On both datasets, compared to I-CFM, \texttt{Latent-CFM}  achieves similar levels of FID ($\sim 3.55$ on CIFAR10, $\sim 7$ on ImageNet) with 50\% fewer training steps. 
% \YS{maybe not use ``significant'' here?}
Note that the best FID for \texttt{Latent-CFM} is $\sim 3.467$, which is lower than the final I-CFM FID $\sim 3.561$ and is the minimum across methods and solvers (Table~\ref {tab:uncond_generation}) on CIFAR10. On CIFAR10, \texttt{Latent-CFM} with a pretrained encoder shows better ($\sim 7\%$ lower FID across steps)
% \YS{might be better to use quantities (e.g., percentage improvement) than qualitative descriptors} 
efficiency than a jointly trained model. On ImageNet, we observe a significant decrease in training speed to $1.3$ steps/second for \texttt{Latent-CFM} with a jointly trained encoder, from $2.4$ steps/second using a pretrained encoder. This further demonstrates the efficiency of \texttt{Latent-CFM} compared to VRFM (which jointly trains an encoder with a larger set of inputs) on large datasets. We add the results for the \texttt{Latent-CFM} with a jointly trained encoder in Appendix~\ref{sec:jont_imnet}.

\paragraph{Generation quality} We evaluate the samples using the Fréchet inception distance (FID) \cite{parmar2021cleanfid}, which quantifies the quality and diversity of the generated images. Table~\ref{tab:uncond_generation} shows the parameter count and the FID of the unconditional generation for the models on CIFAR10, MNIST, and ImageNet for different solvers and different numbers of integration steps. All models are evaluated at their final training step. 
% Compared to the I-CFM and OT-CFM, and VRFM, \texttt{Latent-CFM} achieves lower FID with adaptive solver on CIFAR10.
We observe that \texttt{Latent-CFM} variants consistently outperform the I-CFM and OT-CFM across datasets and solvers in terms of FID. On MNIST, two variants of \texttt{Latent-CFM} show similar FID. On CIFAR10, \texttt{Latent-CFM} with a pretrained encoder achieves lower FID (by $\sim 7\%$) compared to the jointly trained model, highlighting the benefit of a pretrained encoder model. \texttt{Latent-CFM} sampling uses features learned from the training data. Sec.~\ref{sec:generalization} shows that this does not prohibit our approach from generalizing beyond the training data.
% Meanwhile, \texttt{Latent-CFM} shows the lowest FID with a 100-step Euler solver and is similar to the best FID (with I-CFM) obtained with 1000 steps. On MNIST, across solvers, \texttt{Latent-CFM} exhibits the lowest FID values.

With our best efforts, we were unable to reproduce the VRFM FID numbers on CIFAR10 from \cite{guo2025variational}. Table.~\ref{tab:vrfm_cifar} compares the final step FID with \texttt{dopri5} solver for the two VRFM models with our approach and with the two best VRFM models reported in \cite{guo2025variational}. We observe that \texttt{Latent-CFM} performs similarly to the best-performing VRFM model from \cite{guo2025variational} (with $\sim 20\%$ less time in terms of GPU hours). We also observe in our implementation of VRFM that simplifying the input to the model results in lower FID.   

% In Fig~\ref{fig:combinedTrajectories}, we show 10 random sample trajectories of IC-CFM and KDE-CFM methods trained on the MNIST data set. The learned prior using KDE has enabled the flow matching step to find the structure of the image early in the trajectory.

% \begin{figure*}[htp]
% \centering
% \begin{subfigure}[b]{0.48\textwidth} % Width specified here
%     \centering
%     \includegraphics[width=1.2\textwidth]{img/Gaussian_trajectory_99.png}
%     \caption{Sample trajectory using IC-CFM}
%     \label{fig:gaussianTrajectory} % Added label for subfigure
% \end{subfigure}
% %
% \begin{subfigure}[b]{0.48\textwidth} % Width specified here
%     \centering
%     \includegraphics[width=1.2\textwidth]{img/Base_KNIFE_trajectory_99.png}
%     \caption{Sample trajectory using KDE-CFM}
%     \label{fig:kdeFM} % Label retained
% \end{subfigure}
% \begin{subfigure}[b]{0.48\textwidth} % Width specified here
%     \centering
%     \includegraphics[width=1.2\textwidth]{img/REMEDI_trajectory_99.png}
%     \caption{Sample trajectory using \texttt{REMEDI-CFM}}
%     \label{fig:remediFM} % Label retained
% \end{subfigure}
% \caption{Comparison of sample trajectories using IC-CFM, KDE-CFM, and \texttt{REMEDI-CFM} trained on MNIST.}
% \label{fig:combinedTrajectories} % Label for the whole figure
% \end{figure*}

\subsection{Generation of 2D Darcy Flow}
% \AS{Maybe only report one variant of LCFM?}
\begin{figure}[t]
    \centering
    \includegraphics[width=0.8\linewidth]{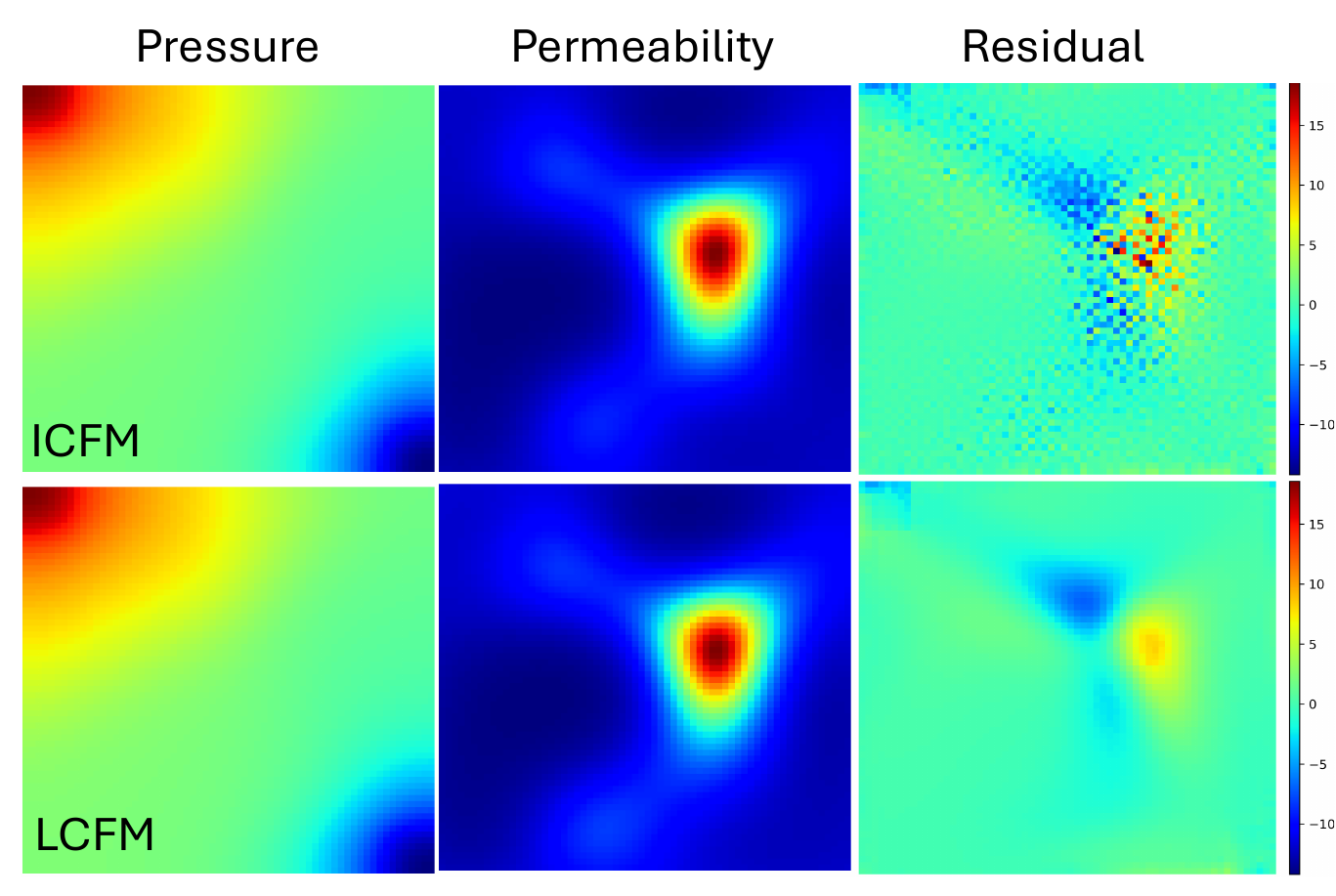}
    \caption{\small Plot showing a generated sample from I-CFM (top row) and \texttt{Latent-CFM} (bottom row) models trained on 10K samples generated by solving Darcy Flow equations \cite{jacobsen2025cocogen}. Visually generated samples from both models resemble true Pressure and Permeability fields. However, \texttt{Latent-CFM} exhibits a better fit to the Darcy flow equations as measured by the residuals.}\label{fig:darcy_flow} 
\end{figure}
\begin{table}
    \centering
    \begin{tabular}{|c|c|c|c|}
    \hline
        &\multicolumn{2}{c|}{\textbf{I-CFM}} & \texttt{Latent-CFM}  \\
        \hline
      \textbf{\# Params.}   & 35.7M & 68.8M & 35.7M\\
      \hline
     \textbf{Residual median} \(\downarrow\) & 5.922 & 4.921 & \textbf{3.18} \\
     \hline
    \end{tabular}
    \caption{Comparison of the PDE residuals of the generated Darcy flow
    samples (\([K, p]\) pairs) from \texttt{Latent-CFM} and I-CFM with two
    model sizes. \texttt{Latent-CFM} model size is a sum of the vector field parameters and the VAE encoder parameters. The samples generated using \texttt{Latent-CFM}, despite the smaller models, present lower PDE residuals.}
    \label{tab:darcy}
\end{table}
\begin{figure}[t]
    \centering
    \includegraphics[width=.75\linewidth]{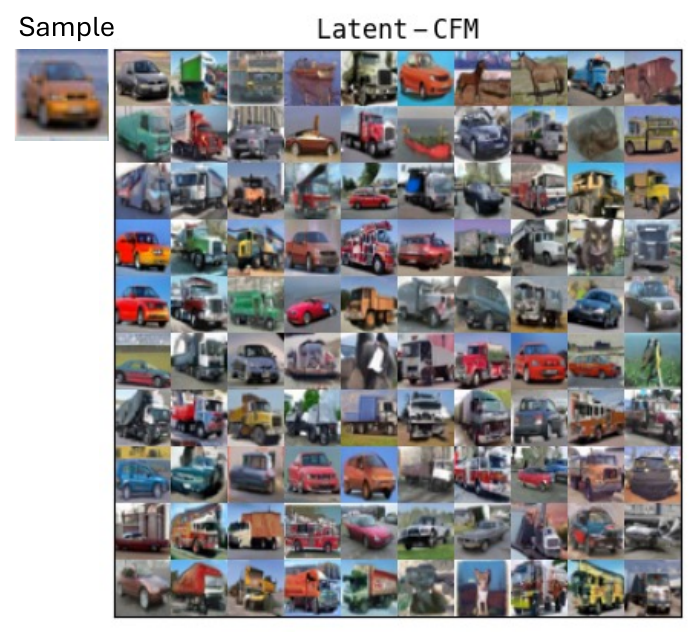}
    % \centering
    % \includegraphics[width=\textwidth]{img/CIFAR10_style.pdf}
    \caption{Conditioning the generation process of \texttt{Latent-CFM} on the features learned from the training samples shows that the framework generates samples by varying the objects while retaining properties like background, colors, object shape, etc. All samples shown have the same feature vector $f$ respectively.}
    \label{fig:cifar_style}
\end{figure}

Beyond the image space, generative models have great potential in advancing scientific computing tasks. Different from images, scientific data must satisfy specific physical laws on top of visual correctness. As a result, generated samples from the unconditional models usually present non-physical artifacts due
to the lack of physics-based structure imposed during learning~\cite{jacobsen2025cocogen, chengHardConstraintGuided2024}. We use \texttt{Latent-CFM} to explore its performance of generating permeability and pressure fields (\(K\) and \(p\)) in 2D Darcy flow and compute the residuals of
the governing equations to evaluate generated samples. Details of data generation and parameter choices are in Appendix~\ref{appdix:darcy}.

We train I-CFM and \texttt{Latent-CFM} on the 10K samples of the 2d Darcy flow process. We adopted the network architecture for the vector field from our CIFAR10 experiments for this data by changing the input convolution to adapt to the Darcy flow data size, which is $(2,64,64)$. We pretrain an embedding VAE model following \cite{rombach2022highresolutionimagesynthesislatent} for 100k steps. We set the latent dimension to be $2$ and $\beta = 0.001$ for the feature extractor.
% along with the CFM training. We also add the results of a \texttt{Latent-CFM} training with a pretrained VAE (Alg.~\ref{alg:vae-cond}) in the supplementary material.
\begin{figure*}[t]
    \centering
    \begin{subfigure}{0.42\textwidth}
        \centering
        \includegraphics[width=\textwidth]{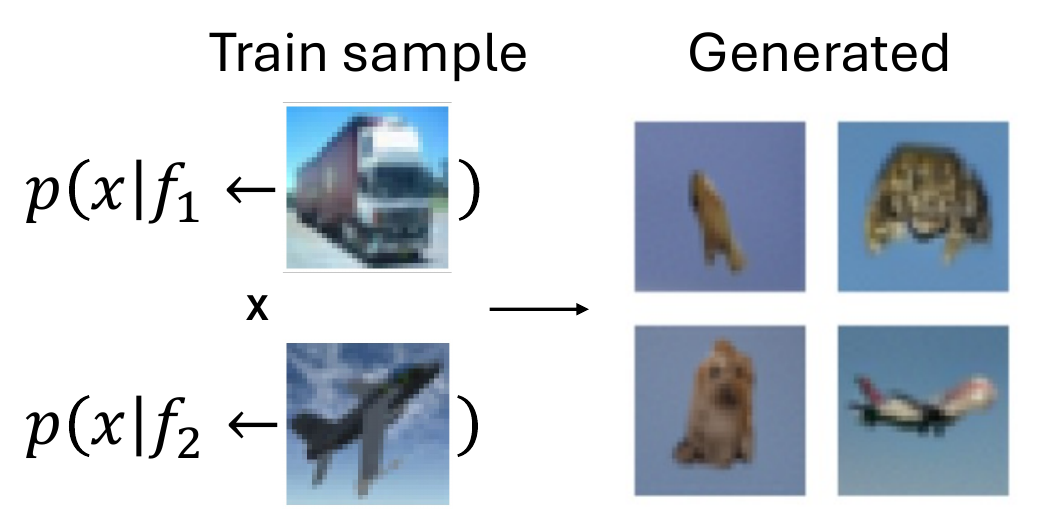}
        \caption{}
    \end{subfigure}
    \hspace{-0.1cm}
    \begin{subfigure}{0.42\textwidth}
        \centering
        \includegraphics[width=\textwidth]{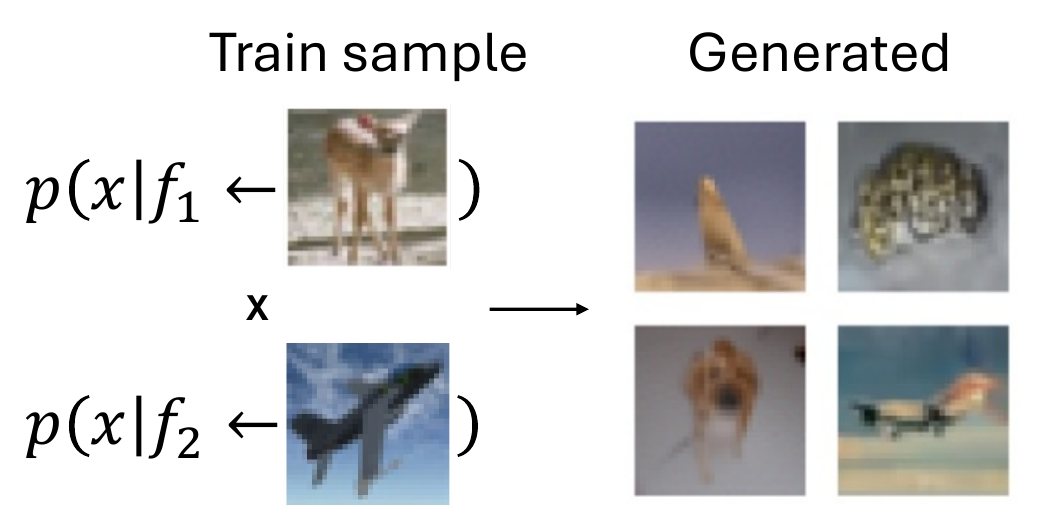}
        \caption{}
    \end{subfigure}
    % \begin{subfigure}{0.3\textwidth}
    %     \centering
    %     \includegraphics[width=\textwidth]{Figures/compose_0_10_v1.pdf}
    %     \caption{}
    % \end{subfigure}
    \caption{Plot showing a selected set of generated samples from the product of two feature-conditioned distributions using \texttt{Latent-CFM}. Changing one of the feature-conditioned distributions in the product, all generated samples share the high-frequency features (such as the skin pattern of the frog), varying the low-frequency features (like color and background).}
    \label{fig:composition}
\end{figure*}

\paragraph{Generation quality} Fig.~\ref{fig:darcy_flow} shows the generated samples from I-CFM and
\texttt{Latent-CFM}. Both models generate visually plausible \([K, p]\) pairs.
However, \texttt{Latent-CFM} resulted in samples with lower residuals, making them more
physically aligned with the governing equations. Table.~\ref{tab:darcy} shows the median mean-squared-residuals for the methods calculated on 500 generated samples and also their parameter counts. We observe that a smaller \texttt{Latent-CFM} (\# parameters $35.7$M) significantly outperforms the large I-CFM model (\# parameters $68.8$M) in terms of the median residual.  Using the 2d latent space, Fig.~\ref{fig:lat_traversal_darcy} (details are in Sec.~\ref{sec:additional_darcy}) shows the latent traversal along each of the two latent coordinates. The figure shows that traversing the latent space generates physically consistent samples. The superiority of \texttt{Latent-CFM} in this dataset motivates us to investigate its performance for other scientific data generation tasks and examine the learned latent space in future work.
% Although \texttt{Latent-CFM} exhibits lower median residual, it generates more extreme outliers as shown in Appendix Fig.~\ref{fig:darcy_res_dist}.

\subsection{Latent space analysis}\label{sec:latent_analysis}
This section analyzes the latent space learned by the \texttt{Latent-CFM} models trained on the CIFAR10 dataset. We investigate (1) the effect of conditioning the \texttt{Latent-CFM} generation on the data features $f \sim q(.|x^{train}_1)$, (2) compositional generation composing multiple feature-conditioned distributions learned from multiple data points. 

\paragraph{Conditional generation by image features} An important property of the \texttt{Latent-CFM} is the ability to generate samples from the distribution $p_1(x|f)$ conditioned on the features $f$. Fig.~\ref{fig:cifar_style} shows 100 generated samples conditioned on the CIFAR10 data sample of a car image (on the left). For all generations, we fix the same latent sample $f \sim q_{\hat{\lambda}}(f|x^{car}_1)$ where $x_1^{car}$ denotes the selected CIFAR10 image and vary the source samples $x_0 \sim N(0,I)$. We observe that \texttt{Latent-CFM} generates different images while retaining properties like color schemes and object shape, etc. 
% Note that \texttt{Latent-CFM} provides access to an approximate posterior likelihood $q_{\hat{\lambda}}(.|x)$, which can be used for classifier guidance \cite{dhariwal2021diffusion} techniques to improve generation quality. We leave this as a future direction.   

\paragraph{Composing feature-conditioned distributions}
Recent works \cite{du2024reducereuserecyclecompositional, bradley2025mechanismsprojectivecompositiondiffusion} in diffusion models have explored generating high-fidelity samples from a composition of class-conditional generative models. We extend this idea to flow matching models for generating samples from a composition of feature-conditioned distributions using \texttt{Latent-CFM}. Given two feature-conditioned densities $p(.|f_1), p(.|f_2)$ where the features $f_1 \sim q(.|x^{train}_1), f_2 \sim q(.|x^{train}_2)$ are extracted from two training samples $(x^{train}_1, x^{train}_2)$, we want to sample from the product distribution $p^1 = \prod_{i=1,2} p(.|f_i)$, which has high likelihood under both feature conditioned distributions. Sec.~\ref {sec:compose_details} describes the details of our inference algorithm.
%\YS{mention why the composition outcome is not ideal? using the explanation and theory in \cite{bradley2025mechanismsprojectivecompositiondiffusion}?}

Fig.~\ref{fig:composition} shows a selected set of generated samples from two products of feature-conditioned distributions. The two products share one common set of features coming from the image of the airplane (bottom row). The Gaussian source samples vary within the generated samples from each product, but are the same between the two. We observe that the generated samples share high-frequency features (such as the skin pattern of the frog) between the two distributions. However, the low-frequency features vary in the generated samples (color and background changes from blue to a mixture of blue and yellow) between the two products. This indicates that the feature extractor helps the \texttt{Latent-CFM} to condition the generation on the low-frequency features from the training data, while the CFM model varies the high-frequency features to generate diverse samples. We provide an expanded set of 100 generated samples from the two product distributions in Fig.~\ref{fig:composition_100}.

\section{Conclusion}
% \AS{Add disentangled latent representation as future work}
Flow matching models generalize the transport paths of diffusion models, thus unifying flow-based generative models. However, existing flow matching and diffusion studies often do not consider the structure of the data explicitly when constructing the flow
% ignore structures in the data when constructing flows 
from source to target distribution. 
% \YS{I think "existing studies ignore the structure in data" is not clear. It may sound like the existing works fail to capture the structure in data. Should we say something like "they do not consider the structure when constructing the flow"? } 
In this study, we present \texttt{Latent-CFM}, a framework that incorporates the underlying clustering structure of the data as latent variables in conditional flow matching. We present training/inference algorithms to adapt popular deep latent variable models into the CFM framework. Using experiments on synthetic and benchmark image datasets, we show that our approach improves (or shows similar) generation quality (FID $\sim 3.5$ on CIFAR10) compared to state-of-the-art CFM models, especially with significantly fewer training steps (with $\sim 50 \%$ in CIFAR10). In addition, we demonstrate the utility of \texttt{Latent-CFM} in generating more physically consistent Darcy flow data than I-CFM. Finally, through latent space analysis, we explore the natural connection of our approach to conditional image generation and compositional generation conditioned on image features like background, color, etc.

One interesting direction for future research could be to tighten the upper bound in Eq.~\ref{eq:latent_CFM_KL_obj} with a data-driven learned prior $\hat{p}(f)$. Based on recent advances in estimating information-theoretic bounds \cite{nilssonremedi}, one can train a learned prior alongside \texttt{Latent-CFM} to learn better latent representations with lower loss value. It would also be interesting to disentangle the latent features, which can further improve control over the generation process.
% This approach can bring about a computational advantage during sampling, where we can avoid reusing the encoder and directly sample from $\hat{p}(f)$. 
In addition, it will be interesting to explore the application of our approach in scientific machine learning. An area of application could be in multifidelity modeling, where using \texttt{Latent-CFM}, one can inform a CFM model trained on a high-fidelity (expensive) simulation dataset (example, fluid dynamics simulations) with latents learned from inexpensive low-fidelity simulated data. Based on our experiment with the Darcy Flow dataset, it could be a promising approach to improve generation performance by satisfying underlying physics constraints.            

\section{Acknowledgments} % will be removed in pdf for initial submission,
						 % (without ‘accepted’ option in \documentclass)
                         % so you can already fill it to test with the
                         % ‘accepted’ class option

The computations were enabled by the computational resources of the Argonne Leadership Computing Facility, which is a DOE Office of Science User Facility supported under Contract DE-AC02-06CH11357, Laboratory Computing Resource Center (LCRC) at the Argonne National Laboratory. 
% The research was supported by the U.S. Department of Energy, Office of Science, Office of Fusion Energy Sciences, under contract DE-AC02-06CH11357 and Advanced Scientific Computing Research, through the SciDAC-RAPIDS2 institute under Contract DE-AC02-06CH11357.
AS, YS, and SM were supported by the U.S. Department of Energy, Office of Science, Advanced Scientific Computing Research, through the SciDAC-RAPIDS2 institute under Contract DE-AC02-06CH11357, and additionally, SM was supported by the Competitive Portfolios For Advanced Scientific Computing Research Project, Energy Efficient Computing: A Holistic Methodology, under Contract DE-AC02-06CH11357.

\bibliography{sample.bib}

\appendix
\renewcommand{\thefigure}{A.\arabic{figure}}
\renewcommand{\thetable}{A.\arabic{table}}
\renewcommand{\thealgorithm}{A.\arabic{algorithm}}
\setcounter{section}{0}
\setcounter{figure}{0}
\setcounter{table}{0}
\setcounter{algorithm}{0}

%%%%%%%%%%%%%%%%%%%%%%%%%%%%%%%%%%%%%%%%%%%%%%%%%%%%%%%%%%%%
\section*{Checklist}

\begin{enumerate}

  \item For all models and algorithms presented, check if you include:
  \begin{enumerate}
    \item A clear description of the mathematical setting, assumptions, algorithm, and/or model. [Yes/No/Not Applicable]
    Yes
    \item An analysis of the properties and complexity (time, space, sample size) of any algorithm. [Yes/No/Not Applicable]
    Yes
    \item (Optional) Anonymized source code, with specification of all dependencies, including external libraries. [Yes/No/Not Applicable]
    Yes
  \end{enumerate}

  \item For any theoretical claim, check if you include:
  \begin{enumerate}
    \item Statements of the full set of assumptions of all theoretical results. [Yes/No/Not Applicable]
    Yes
    \item Complete proofs of all theoretical results. [Yes/No/Not Applicable]
    Yes
    \item Clear explanations of any assumptions. [Yes/No/Not Applicable]
    Yes
  \end{enumerate}

  \item For all figures and tables that present empirical results, check if you include:
  \begin{enumerate}
    \item The code, data, and instructions needed to reproduce the main experimental results (either in the supplemental material or as a URL). [Yes/No/Not Applicable]
    Yes
    \item All the training details (e.g., data splits, hyperparameters, how they were chosen). [Yes/No/Not Applicable]
    Yes
    \item A clear definition of the specific measure or statistics and error bars (e.g., with respect to the random seed after running experiments multiple times). [Yes/No/Not Applicable]
    Yes for W2 distances in Table 1. FIDs have little variation.
    \item A description of the computing infrastructure used. (e.g., type of GPUs, internal cluster, or cloud provider). [Yes/No/Not Applicable]
    Yes.
  \end{enumerate}

  \item If you are using existing assets (e.g., code, data, models) or curating/releasing new assets, check if you include:
  \begin{enumerate}
    \item Citations of the creator If your work uses existing assets. [Yes/No/Not Applicable]
    Yes
    \item The license information of the assets, if applicable. [Yes/No/Not Applicable]
    Not Applicable
    \item New assets either in the supplemental material or as a URL, if applicable. [Yes/No/Not Applicable]
    Yes
    \item Information about consent from data providers/curators. [Yes/No/Not Applicable]
    Not Applicable
    \item Discussion of sensible content if applicable, e.g., personally identifiable information or offensive content. [Yes/No/Not Applicable]
    Not Applicable
  \end{enumerate}

  \item If you used crowdsourcing or conducted research with human subjects, check if you include:
  \begin{enumerate}
    \item The full text of instructions given to participants and screenshots. [Yes/No/Not Applicable]
    Not Applicable
    \item Descriptions of potential participant risks, with links to Institutional Review Board (IRB) approvals if applicable. [Yes/No/Not Applicable]
    Not Applicable
    \item The estimated hourly wage paid to participants and the total amount spent on participant compensation. [Yes/No/Not Applicable]
    Not Applicable
  \end{enumerate}

\end{enumerate}

\clearpage
\appendix
% \thispagestyle{empty}

% Supplementary material: To improve readability, you must use a single-column format for the supplementary material.
\onecolumn
\aistatstitle{Appendix}

\section{Theoretical analysis}
In the following proposition, we adopt a slightly different formalism where we denote between the random variables with capital letters ($X_0, X_1, F$, etc.) and the values they take ($x_0, x_1, f$).
For a random variable $Y$, $\mathcal{L}(Y)$ denotes its law.

\begin{proposition}\label{prop:marginal-preserved}
    Let $(X_0, X_1)$ be drawn under some probability measure $\mathbb{P}$ on a measurable space $(\Omega, \mathcal{F})$, which also carries the latent random variable $F$ in $\mathbb{R}^{d_f}$. %, which is conditionally independent of $X_0$ given $X_1$.
    Assume that $X_1 - X_0$ is integrable and define the process $\{X_t\}_t$ by 
    \begin{equation}\label{eq:cond-straight-dynamics}
        X_t = tX_0 + (1-t)X_1,
    \end{equation}
    and let $\{\mu^f_t\}_t = \{\mathcal{L}(X_t \mid F = f)\}_t$ be its ($F$-conditional) marginal probability path, under $\mathbb{P}$.
    Given an optimally learned $F$-conditional vector field $v^* = v^{*}_{t,f}(x)$, minimizing the loss function in \eqref{eq:latent_CFM_KL_obj}, and its $F$-conditional flow $\phi^{v^*, f}$, we have that its marginal probability path is equal to that of the ground truth process in \eqref{eq:cond-straight-dynamics} $\mathbb{P}$-a.s. in $f$, i.e. $\{\mathcal{L}(\phi^{v^*, f}_t(X_0))\}_{t} = \{\mu^f_t\}_t$ ($F_\#\mathbb{P}$-a.s. in $f$).
    %In particular, $\mu_1 = \mathcal{X_1}$.
    In particular, $\mathcal{L}(\phi^{v^*, f}_1(X_0)) \overset{\mathcal{}}{=} \mathcal{L}(X_1 \mid F = f)$ ($F_\#\mathbb{P}$-a.s. in $f$) and $\phi^{v^*, F}_1(X_0) \overset{\mathcal{L}}{=} X_1$.
    %\todo{It should somehow be presented that $F$ and $X_0$ are independent. Maybe this is actually unneccesary to include.}
\end{proposition}

\begin{proof}
    We follow the proof ideas in \cite{liu2022flow} and \cite{guo2025variational}.
    We want to see that $(\mu, v^*)$ satisfies the continuity equation
    \begin{equation}\label{eq:continuity-equation}
        \dot{\mu}^f_t + \nabla \cdot (\mu^f_t v^*_{t, f}) = 0.
    \end{equation}
    The meaning of \eqref{eq:continuity-equation} is only formal, as $\mu^f_t$ may not even have a density.
    The \textbf{definition} of \eqref{eq:continuity-equation} is that for any \textit{test function} $h: \mathbb{R}^d \to \mathbb{R}$, i.e. $h$ is smooth and compactly supported in $\mathbb{R}^d$,
    % \begin{equation}\label{eq:continuity-equation-definition-1}
    %     \frac{d}{dt}\mathbb{E}[h(X_t)] = \mathbb{E}[\langle \nabla h(X_t), v^*(t, X_t) \rangle ]
    % \end{equation}
    \begin{equation}\label{eq:continuity-equation-definition-2}
        \frac{d}{dt}\int_{\mathbb{R}^d} h \, d\mu^f_t = \int_{\mathbb{R}^d} \langle \nabla h, v^*_{t,f} \rangle \, d\mu^f_t
    \end{equation}
    holds (in the sense of distributions on $(0, 1)$) \cite[p. 169-170]{ambrosio2008gradient}.
    Under some regularity conditions on $v^*$ \cite[Proposition 8.1.8]{ambrosio2008gradient}, we know that the theorem follows if we can show \eqref{eq:continuity-equation-definition-2}.

    Here, in fact, the derivative on the left hand side of \eqref{eq:continuity-equation-definition-2} exists in the classical sense, and the differentiation can be moved under the integral sign.
    This is seen by noting that with $g(t, \omega) \coloneq h(tX_0(\omega) + (1-t)X_1(\omega))$, the conditions for doing so in $\frac{d}{dt} \int g(t, \omega) \, \mathbb{P}_f(d\omega)$ are fulfilled, see e.g. \cite[Ch. A5]{durrett2019probability}.
    Most importantly, $\frac{d}{dt} g(t, \omega) = \langle \nabla h(tX_0 + (1-t)X_1), X_1 - X_0 \rangle_{\mathbb{R}^d}(\omega) = \langle \nabla h(X_t), \dot{X}_t \rangle_{\mathbb{R}^d}(\omega)$, where $h$ and its derivatives are bounded, and $\dot{X}_t = X_1 - X_0$ is integrable and thus conditionally integrable for almost all $f$.
    We get, taking $\mathbb{P}_{f}$ to be a regular conditional probability measure for $F = f$, known to exist via the disintegration theorem \cite[Theorem~3.4]{kallenberg2021foundations},
    \begin{equation}\label{eq:continuity-equation-definition-proof-1}
        \frac{d}{dt}\int_{\mathbb{R}^d} h \, d\mu^f_t = \int_{\mathbb{R}^d} \langle \nabla h(X_t), \dot{X}_t \rangle_{\mathbb{R}^d}(\omega) \mathbb{P}_f(d\omega) = \mathbb{E}^{\mathbb{P}_f}[\langle \nabla h(X_t), \dot{X}_t \rangle_{\mathbb{R}^d}] = \ldots
    \end{equation}
    Further, by using the tower property to condition on $X_t$,
    \begin{equation}\label{eq:continuity-equation-definition-proof-2}
        \ldots = \mathbb{E}^{\mathbb{P}_f}[\mathbb{E}^{\mathbb{P}_f}[\langle \nabla h(X_t), \dot{X}_t \rangle_{\mathbb{R}^d} \mid X_t]] = \mathbb{E}^{\mathbb{P}_f}[\langle \nabla h(X_t), \mathbb{E}^{\mathbb{P}_f}[\dot{X}_t \mid X_t] \rangle_{\mathbb{R}^d}] 
    \end{equation}
    But for almost all $f$, we have that $\mathbb{E}^{\mathbb{P}_f}[\dot{X}_t \mid X_t] = v^*_{t, f}(X_t)$, since $v^*$ is optimal for \eqref{eq:latent_CFM_KL_obj}, whose minimizer (for a fixed encoder $q_{\lambda_{final}}$) $v^*_{t,f}(x)$ is
    \begin{equation}\label{eq:continuity-equation-definition-proof-3}
        \mathbb{E}[u_t(x|X_0,X_1) \mid X_t = x, F = f] = \mathbb{E}[X_1 - X_0 \mid X_t = x, F = f] = \mathbb{E}^{\mathbb{P}_f}[\dot{X}_t \mid X_t = x].
    \end{equation}
    This in \eqref{eq:continuity-equation-definition-proof-1} and \eqref{eq:continuity-equation-definition-proof-2} gives
    \begin{equation}
        \frac{d}{dt}\int_{\mathbb{R}^d} h \, d\mu^f_t = \mathbb{E}^{\mathbb{P}_f}[\langle \nabla h(X_t), v^*_t(X_t) \rangle_{\mathbb{R}^d}] = \int_{\mathbb{R}^d} \langle \nabla h, v^*_{t,f} \rangle_{\mathbb{R}^d} \, d\mu^f_t,
    \end{equation}
    which finishes the proof of the main statement.
    The last part of the proposition can be seen from:
    \begin{equation}
        \mathbb{P}(\phi^{v^*, F}_1(X_0) \in A) = \mathbb{E}[\mathbb{P}(\phi^{v^*, F}_1(X_0) \in A \mid F)] = \mathbb{E}[\mathbb{P}(X_1 \in A \mid F)] = \mathbb{P}(X_1 \in A).
    \end{equation}
\end{proof}

\section{Relation with Variational Rectified Flow}
\label{sec:vrfm}
Recent work in variational rectified flow matching (VRFM) \cite{guo2025variational} has studied the effect of a mixture model of the vector field $u_t(x)$ induced by a latent variable. In Eq.\ref{eq:cond_flow}, we observe that the CFM objective function can be viewed as a log-likelihood of a Gaussian distribution model for $u_t(x) \sim N(u_t; v_\theta(x,t),I)$. In VRFM, we model the vector field by a mixture model induced by a latent variable $z \sim p(z)$,
\begin{align}\label{eq:mixture_vec}
    p(u_t|x_t,t) = \int p_\theta(u_t|x_t,t,z) p(z)df
\end{align}
In VRFM, given $z$, the conditional density $p_\theta(u_t|x_t,t,z)$ are assumed to be $N(u_t;v_{\theta}(x,t,z),I)$. To learn the latent variable $z$, the authors use a recognition model $q_\phi(z|x_0,x_1,x_t,t)$ or encoder. The parameters of the encoder and the learned vector field are learned jointly by optimizing a VAE objective,
\begin{align}\label{eq:VRFM_loss}
    \log p(u_t|x_t,t) &\geq \E_{z\sim q_{\phi}} \left[\log p_\theta(u_t|x_t,t,z)\right] \nonumber \\ &- D_{\mathrm{KL}}(q_\phi(z|x_0,x_1,x_t,t)||q(z)) 
\end{align}
A key observation in the VRFM objective in Eq~\ref{eq:VRFM_loss} is that the encoder model $q_\phi$ depends on $(x_0,x_1,x_t,t)$ which dynamically changes with time $t$. However, the generative model $q(z)$ is static and equal to $N(0,I)$. To generate samples from VRFM, we sample $f\sim N(0,I)$ once and solve Eq.~\ref{eq:vec_field} using the learned vector field $v_{\theta}(x,t,z)$ fixing $z$. 

In \texttt{Latent-CFM}, we propose learning a static encoder model $q_\phi$ from the data $x_1$ and optimize the Eq~\ref{eq:VRFM_loss} w.r.t the encoder and the vector field parameters. This change has the following advantages, (1) We can generate samples conditioned on the variables learned from the data distribution $p_1(x)$, and (2) We can use pre-trained feature extractors and fine-tune them to optimize the CFM loss.

% \section{\texttt{Latent-CFM} with Gaussian Mixture Models}

% GMMs use mixture of Gaussian kernels to model the data distribution. An $M$ component GMM is defined as,
% \begin{align}\label{eq:GMM}
%     q_\lambda(x) = \sum_{j=1}^{M} w_j N(x; \mu,\sigma^2I); \quad \lambda = (\mu, \sigma)
% \end{align}

\section{Variational AutoEncoders}
\label{sec:VAE}
VAE is a popular deep latent variable model that assumes a latent variable $f$ is governing the data distribution $p_1(x_1) = \int p(f)p(x_1|f)df$. The posterior distribution $p(f|x)$ is intractable and hence is approximated by a variational distribution $q_{\lambda}(f|x)$ which is then learned by optimizing an ELBO:
\begin{align}\label{eq:loss_VAE}
    \mathcal{L}_{VAE} &= -\E_{p_{data}(x_1)} \big[\E_{q_{\lambda}(f|x_1)} \log p_{\psi}(x_1|f) + D_{KL}(q_{\lambda}(f|x_1)||p(f)) \big]
\end{align}
where, $q_{\lambda}(f|x_1)$ and $p_{\psi}(x_1|f)$ are parameterized by an encoder and a decoder neural network respectively and $p(f)$ is assumed to be $N(0,I)$. The variational distribution is commonly assumed to be multivariate Gaussian with mean $\mu$ and a diagonal covariance matrix $\sigma^2I$. The encoder network $q_{\lambda}(f|x_1)$ outputs the parameters $\mu$ and $\sigma^2$.  VAEs are successful in generative modeling applications in a variety of domains. However, they often suffer from low generation quality in high-dimensional problems.

\section{Darcy Flow Dataset}\label{appdix:darcy}
The Darcy flow equation describes the fluid flowing through porous media. With
a given permeability field \(K(x)\) and a source function \(f_s(x)\), the pressure
\(p(x)\) and velocity \(u(x)\) of the fluid, according to Darcy's law, are governed
by the following equations
\begin{equation}\label{eqn:darcy}
    \begin{aligned}
        u(x) &= -K(x)\nabla p(x), \quad  x \in \Omega\\
        \nabla \cdot u(x) &= f_s(x),\quad  x \in \Omega\\
        u(x) \cdot n(x) &= 0, \quad  x \in \partial \Omega\\
        \int_{\Omega}p(x)dx &= 0,
    \end{aligned}
\end{equation}
where \(\Omega\) denotes the problem domain and \(n(x)\) is the 
outward unit vector normal to the boundary. Following the problem 
set up in~\cite{jacobsen2025cocogen}, we set the source term as 
\begin{equation}
    f_s(x) = 
    \begin{cases}
        r, \quad \mid x_i - 0.5w\mid \leq 0.5w, \, i =1, 2\\
        -r, \quad \mid x_i -1 + 0.5w\mid \leq 0.2 w, \, i= 1,2\\
        0, \quad \text{otherwise}
    \end{cases}, 
\end{equation}
and sample \(K(x)\) from a Gaussian random field, \(K(x) = \exp(G(x)), \,
G(\cdot) \sim \mathcal{N}(\mu, k(\cdot, \cdot))\), where the covariance function
is \(k(x, x^{\prime}) = \exp(- \frac{\Vert x - x^{\prime}\Vert_2}{l})\). 

Using the finite difference solver, we created a dataset containing 10,000 pairs of
\([K,p]\) and trained I-CFM and \texttt{Latent-CFM} models to generate new
pairs. The generated samples are expected to follow (\ref{eqn:darcy}).
Therefore, we use the residual of the governing equation to evaluate a sample 
quality. In particular, for each generated sample, we compute
\begin{equation}\label{eqn:darcy_residual}
\begin{aligned}
    R(x) &= f_s(x) + \nabla \cdot [K(x) \nabla p(x)]\\
      &= f_s(x) + K(x)\frac{\partial^2p(x)}{\partial x_1^2} + \frac{\partial K(x)}{\partial x_1}\frac{\partial p(x)}{\partial x_1}\\
      &+ K(x)\frac{\partial^2p(x)}{\partial x_2^2} + \frac{\partial K(x)}{\partial x_2}\frac{\partial p(x)}{\partial x_2}.
\end{aligned}
\end{equation}
The partial derivatives in (\ref{eqn:darcy_residual}) are approximated using central finite differences.

We generated 500 \([K, p]\) pairs from trained I-CFM and \texttt{Latent-CFM} 
models and computed the residual according to (\ref{eqn:darcy_residual}) for evaluation.
Figure~\ref{fig:darcy_flow} shows sample examples of Darcy flow from the two models, and
Table \ref{tab:darcy} presents the comparison of the median of sample residuals between the two models.

% \begin{figure}
%     \centering
%     \includegraphics[width=0.9\linewidth]{Figures/boxplot_residual_mse_medians.png}
%     \caption{Distribution of the residual MSE of I-CFM and \texttt{Latent-CFM} shows that the latter produces the overall lower median residual MSE (numbers beside each box plot). However, \texttt{Latent-CFM} seems to generate more extreme outlier residuals than I-CFM.}
%     \label{fig:darcy_res_dist}
% \end{figure}

\section{Datasets}
\label{sec:datasets}
We describe the details of the datasets used in this study.

\subsection{Triangle dataset}
The triangular dataset shares the same structural design as in \cite{pichler2022differentialentropyestimatortraining, nilssonremedi}. For any dimension $d > 1$, it is constructed as the $d$-fold product of a multimodal distribution with $k$ modes (as illustrated in Fig. 2 of \cite{pichler2022differentialentropyestimatortraining} for $k=10$), resulting in a distribution with $k^d$ modes. For the experiments in the main paper, we set $k=4$ and $d=2$, creating $16$ modes over the 2d plane.

\subsection{MNIST and CIFAR10}
We download and use MNIST \cite{lecun1998gradient} and CIFAR10 \cite{krizhevsky2009learning} datasets using the classes \texttt{torchvision.datasets.MNIST}, and \texttt{torchvision.datasets.CIFAR10} from the PyTorch library \cite{paszke2019pytorchimperativestylehighperformance} respectively. On MNIST, we normalize the data using the mean $[0.5,0.5]$ and the standard deviation $[0.5,0.5]$. On CIFAR10, we use random horizontal flipping of the data and normalize using the mean $[0.5,0.5,0.5]$, and the standard deviation $[0.5,0.5,0.5]$.

% \subsection{CIFAR10}

\subsection{Darcy Flow data}
The Darcy flow equation describes the fluid flowing through porous media. With
a given permeability field \(K(x)\) and a source function \(f_s(x)\), the pressure
\(p(x)\) and velocity \(u(x)\) of the fluid, according to Darcy's law, are governed
by the following equations
\begin{equation}\label{eqn:darcy}
    \begin{aligned}
        u(x) &= -K(x)\nabla p(x), \quad  x \in \Omega\\
        \nabla \cdot u(x) &= f_s(x),\quad  x \in \Omega\\
        u(x) \cdot n(x) &= 0, \quad  x \in \partial \Omega\\
        \int_{\Omega}p(x)dx &= 0,
    \end{aligned}
\end{equation}
where \(\Omega\) denotes the problem domain and \(n(x)\) is the 
outward unit vector normal to the boundary. Following the problem 
set up in~\cite{jacobsen2025cocogen}\footnote{We use the same data generation code available at \url{https://github.com/christian-jacobsen/CoCoGen/blob/master/data_generation/darcy_flow/generate_darcy.py}}, we set the source term as 
\begin{equation}
    f_s(x) = 
    \begin{cases}
        r, \quad \mid x_i - 0.5w\mid \leq 0.5w, \, i =1, 2\\
        -r, \quad \mid x_i -1 + 0.5w\mid \leq 0.2 w, \, i= 1,2\\
        0, \quad \text{otherwise}
    \end{cases}, 
\end{equation}
and sample \(K(x)\) from a Gaussian random field, \(K(x) = \exp(G(x)), \,
G(\cdot) \sim \mathcal{N}(\mu, k(\cdot, \cdot))\), where the covariance function
is \(k(x, x^{\prime}) = \exp(- \frac{\Vert x - x^{\prime}\Vert_2}{l})\) and \(G(x)=\mu +\sum_{i=1}^s\sqrt{\lambda_i}\theta_i\phi_i(x)\), where where \(\lambda_i\) and \(\phi_i(x)\) are eigenvalues and eigenfunctions of the covariance function sorted by decreasing \(\lambda_i\), and 
\(\theta_i \sim \mathcal{N}(0, \mathrm{I})\).

We sample permeability fields and solve for the pressure fields which results in
10,000 \([K, p]\) pairs with \(r=10\), \(w=0.125\), and \(s=16\) on \(64 \times 64\) grids for model
training. During training, we
standardize both the permeability and pressure fields using \(\mu_K= 1.1491\),
\(\sigma_K=7.8154\), and \(\mu_p=0.0\), \(\sigma_p=0.0823\).

% We generated 500 \([K, p]\) pairs from trained I-CFM and \texttt{Latent-CFM} 
% models and computed the residual according to (\ref{eqn:darcy_residual}) for evaluation.
% Figure~\ref{fig:darcy_flow} shows sample examples of Darcy flow from the two models, and
% Table \ref{tab:darcy} presents the comparison of the median of sample residuals between the two models.

% \section{\texttt{Latent-CFM} algorithms}
\begin{algorithm}
\caption{\texttt{Latent-CFM} training}\label{alg:vae-cond}
\begin{algorithmic}[1]
\STATE Given $n$ sample $(x^1_1, ..., x^1_n)$ from $p_1(x)$, regularizer $\beta$;
\IF{no pretrained VAE available}
\STATE Train \texttt{VAE} using $(x^1_1, ..., x^1_n)$ optimizing Eq.~\ref{eq:loss_VAE}
\STATE Save the encoder $q_{\hat{\lambda}}(.|x_1)$ 
\ENDIF
% \STATE $\hat{\Q} \leftarrow$ \texttt{REMEDI}$(\Q)$.
\STATE Initialize $v_{\theta}(\cdot, \cdot,\cdot)$ and last encoder layer parameters $\lambda_{last}$
\FOR{$k$ steps}%\Comment{KNIFE step}
% \STATE For $x_i \in (x_1^1, \dots, x_n^1)$, let $c_i = \operatorname*{argmax}_{m} \hat{q}_m(x_i)$
\STATE Sample latent variables $f_i \sim q_{\lambda_{last}}(f|x^1_i)$ for all $i=1,...,n$
\STATE Sample $(x^0_1, ..., x^0_n)$ from $\mathcal{N}(0, I)$ and noise levels $(t_1,...,t_n)$ from $Unif(0,1)$ and compute $(u_{t_1}(.|x_0,x_1),...,u_{t_n}(.|x_0,x_1))$
\STATE compute $v_{\theta}(x_i^{t_i}, f_i,t_i)$ where $x_i^{t_i}$ is the corrupted $i$-th data at noise level ${t_i}$
\STATE Compute $\nabla \mathcal{L}_{\texttt{Latent-CFM}}$ and update $\theta, \lambda_{last}$
\ENDFOR
%\STATE Draw $m$ samples $(\Tilde{x}_1, ..., \Tilde{x}_m)$ from KNIFE
\STATE \textbf{return} $v_{\theta}(\cdot, \cdot,\cdot), q_{\hat{\lambda}}(.|x)$
\end{algorithmic}
\end{algorithm}

\begin{algorithm}[t]
\caption{\texttt{Latent-CFM} inference}\label{alg:vae-cond_inference}
\begin{algorithmic}[1]
\STATE Given sample size $K$, trained $v_{\hat{\theta}}(.,.,.)$ and $q_{\hat{\lambda}}(.|x_1)$, number of ODE steps $n_{ode}$
% \STATE $\hat{\Q} \leftarrow$ \texttt{REMEDI}$(\Q)$.
\STATE Select $K$ training samples $(x_1^{train},...,x_K^{train})$
\STATE Sample latent variables $f_i \sim q_{\hat{\lambda}}(f|x_i^{train})$ for all $i=1,...,K$
\STATE Sample $(x^0_1, ..., x^0_n)$ from $\mathcal{N}(0, I)$
\STATE $h \leftarrow \frac{1}{n_{ode}}$ 
\FOR{$t=0,h,...,1-h$ and $i=1,...,K$}
\STATE $x_i^{t+h} =$ ODEstep($v_{\hat{\theta}}(x_i^t,f_i,t),x_i^t$)
\ENDFOR
\STATE \textbf{return} Samples $(x_1^1,...,x_K^1)$
\end{algorithmic}
\end{algorithm}

\section{\texttt{Latent-CFM} with Gaussian Mixture Models}
\label{sec:lcfm_gmm}
Gaussian mixture models (GMM) use a mixture of Gaussian kernels to model the data distribution. An $M$ component GMM is defined as,
\begin{align}\label{eq:GMM}
    q_{\lambda}(x) = \sum_{j=1}^{M} w_j N(x; \mu_j,\Sigma_j), \quad \sum_{j=1}^{M}w_j = 1
\end{align}
where, $\lambda = \{\mu_j,\Sigma_j,w_j: j = 1,...,M\}$ are the GMM parameters. GMMs are popular in density estimation tasks and are consistent estimators of the entropy of a probability distribution under certain assumptions (Theorem 1 in \cite{pichler2022differentialentropyestimatortraining}). Since the mixture components are Gaussian, one can easily sample from a fitted GMM. However, estimation of GMMs requires a very large number of training samples for moderately large dimensional data and is prone to overfitting (Fig.1 in \cite{nilssonremedi}).

We explore GMM as an alternative to the VAE as a feature extractor in \texttt{Latent-CFM}. We follow \cite{pichler2022differentialentropyestimatortraining} to train the GMMs using the cross-entropy loss function,
\begin{align}\label{eq:loss_GMM}
    \mathcal{L}_{GMM} = - \E_{p_1(x)} \log q_{\lambda}(x)
\end{align}
% We use GMMs for our 2d synthetic data experiments.  
Alg.~\ref{alg:GMM-cond} describes the \texttt{Latent-CFM} training using GMMs. First, we pretrain a GMM by optimizing Eq.~\ref{eq:loss_GMM}. Following \cite{jia2024structured}, during the CFM training, we assign each sample $x^1_i$ a cluster membership id $c_i$ based on the mixture component, which shows the maximum likelihood calculated for the data sample. These ids are passed to the learned vector field $v_\theta(.,.,.)$ as the conditioning variables. The rest of the training is similar to Alg. 1 of the main paper. Alg.~\ref{alg:GMM-cond} does not involve a finetuning of the encoder during the CFM training, therefore, we drop the KL term in the \texttt{Latent-CFM} loss and optimize Eq. 10 of the main paper. 

\begin{algorithm}
\caption{\texttt{Latent-CFM} w GMM training}\label{alg:GMM-cond}
\begin{algorithmic}[1]
\STATE Given $n$ sample $(x^1_1, ..., x^1_n)$ from $p_1(x)$;
\IF{no pretrained GMM available}
\STATE Train GMM using $(x^1_1, ..., x^1_n)$ optimizing Eq.~\ref{eq:loss_GMM}
\STATE Save the GMM $q_{\hat{\lambda}}(.)$ 
\ENDIF
% \STATE $\hat{\Q} \leftarrow$ \texttt{REMEDI}$(\Q)$.
\STATE Initialize $v_{\theta}(\cdot, \cdot,\cdot)$
\FOR{$k$ steps}%\Comment{KNIFE step}
% \STATE For $x_i \in (x_1^1, \dots, x_n^1)$, let $c_i = \operatorname*{argmax}_{m} \hat{q}_m(x_i)$
\STATE Calculate the cluster memberships $c_i = \underset{{j=1,..,M}}{argmax}(N(x^1_i;\hat{\mu_j},\hat{\Sigma}_j)), i = 1,...,n$ 
\STATE Sample $(x^0_1, ..., x^0_n)$ from $\mathcal{N}(0, I)$ and noise levels $(t_1,...,t_n)$ from $Unif(0,1)$ and compute $(u_{t_1}(.|x_0,x_1),...,u_{t_n}(.|x_0,x_1))$
\STATE compute $v_{\theta}(x_i^{t_i}, c_i,t_i)$ where $x_i^{t_i}$ is the corrupted $i$-th data at noise level ${t_i}$
\STATE Compute $\nabla \mathcal{L}_{\texttt{Latent-CFM}}$ in Eq. 10 (main paper) and update $\theta$
\ENDFOR
%\STATE Draw $m$ samples $(\Tilde{x}_1, ..., \Tilde{x}_m)$ from KNIFE
\STATE \textbf{return} $v_{\theta}(\cdot, \cdot,\cdot), q_{\hat{\lambda}}(.)$
\end{algorithmic}
\end{algorithm}

Alg.~\ref{alg:GMM-cond_inference} describes the inference steps using the \texttt{Latent-CFM} with GMM. Given a budget of $K$ samples, we draw the cluster membership ids $(c_1,...,c_K)$ from the distribution $Categorical(w_1,...,w_K)$. This step helps maintain the relative proportion of the clusters in the generated sample set according to the estimated GMM. The rest of the inference steps are similar to Alg. 2 from the main paper.
\begin{algorithm}
\caption{\texttt{Latent-CFM} w GMM inference}\label{alg:GMM-cond_inference}
\begin{algorithmic}[1]
\STATE Given sample size $K$, trained $v_{\hat{\theta}}(.,.,.)$ and $q_{\hat{\lambda}}(.)$, number of ODE steps $n_{ode}$
% \STATE $\hat{\Q} \leftarrow$ \texttt{REMEDI}$(\Q)$.
\STATE Select $K$ random cluster memberships $(c_1,...,c_K)$ from the distribution $Categorical(\hat{w}_1,...,\hat{w}_K)$ 
\STATE Sample $(x^0_1, ..., x^0_n)$ from $\mathcal{N}(0, I)$
\STATE $h \leftarrow \frac{1}{n_{ode}}$ 
\FOR{$t=0,h,...,1-h$ and $i=1,...,K$}
\STATE $x_i^{t+h} =$ ODEstep($v_{\hat{\theta}}(x_i^t,c_i,t),x_i^t$)
\ENDFOR
\STATE \textbf{return} Samples $(x_1^1,...,x_K^1)$
\end{algorithmic}
\end{algorithm}

\section{Implementation details}
\label{sec:implementation}
We provide implementation details of \texttt{Latent-CFM} and other methods used in the experiments section. The full codebase is available at \url{https://anonymous.4open.science/r/Latent_CFM-66CF/README.md}. We closely follow the implementation in the \cite{tong2024improving} repository\footnote{https://github.com/atong01/conditional-flow-matching.git}.

\paragraph{Synthetic data} All CFM training on the 2d triangle dataset has the same neural network architecture for the learned vector field. The architecture is a multi-layered perceptron (MLP) with three hidden layers with \texttt{SELU} activations. For I-CFM and OT-CFM, the input to the network is $(x_t, t)$, and it outputs the learned vector field value with the same dimension as $x_t$.

\texttt{Latent-CFM} and VRFM training have an additional VAE encoder that learns the latent representations during training. For VRFM, we consider the same MLP architecture used for the vector field as the VAE encoder, with the final layer outputs 2d mean and 2d log-variance vectors. The input to the encoder involves the tuple $(x_0,x_1,x_t,t)$ as recommended by \cite{guo2025variational}, which outputs the latent variable $z_t$ and is added to the input tuple $(x_t,z_t,t)$ and passed to the vector field network. The pretrained VAE in \texttt{Latent-CFM} has the same encoder and a one-hidden-layer MLP with \texttt{SELU} activation as a decoder. In the CFM training in \texttt{Latent-CFM}, we fix the pretrained VAE encoder and add a trainable layer, which predicts the mean and the log-variance. The VAE encoder in \texttt{Latent-CFM} takes the input $x_1$ and outputs the latent variable $f$, which is then added to the input tuple $(x_t,f,t)$ in the CFM training. For both VRFM and \texttt{Latent-CFM} we fix the KL regularization parameter $\beta=0.01$.

In \texttt{Latent-CFM} with GMM, we consider a diagonal covariance matrix $\Sigma_j = \sigma_j^2I$ for each Gaussian component and set the number of components $K=16$. The CFM architecture is the same as above.

\paragraph{MNIST and CIFAR10} All models used the same U-Net architecture from \cite{tong2024improving} on MNIST and CIFAR10. The hyperparameters of the learned vector field are changed depending on the dataset. For I-CFM and OT-CFM, the model takes the input $(x_t,t)$ where both variables are projected onto an embedding space and concatenated along the channel dimension and passed through the U-Net layers to output the learned vector field.

In \texttt{Latent-CFM}, we use the pretrained VAE model available for MNIST\footnote{https://github.com/csinva/gan-vae-pretrained-pytorch}, and CIFAR10\footnote{https://github.com/Lightning-Universe/lightning-bolts/}. We take the latent encodings from the last encoder layer and add a trainable MLP layer to output the mean and log-variance of the latent space. Using the reparameterization trick \cite{kingma2022autoencodingvariationalbayes}, we sample the latent variable and project it to the embedding space of the CFM model using a single trainable MLP layer. These feature embeddings are added (see Fig. 2 of the main paper) to the time embeddings and passed to the U-Net. For encoder joint training with \texttt{Latent-CFM}, we keep the same architecture as in the above codebases.

We have also implemented the VRFM model on the CIFAR10 dataset. Following \cite{guo2025variational}, we adopted the downsampling layers of the vector field network to reduce the input to a latent of size $(512,1)$. The addition of the latents to the vector field network follows the same architecture as described in Fig.~\ref{fig:schematic}.  

\paragraph{Darcy Flow} We follow the same model architecture used for CIFAR10 for the I-CFM training in the Darcy Flow dataset. On this dataset, for \texttt{Latent-CFM}, we train a VAE encoder following \cite{rombach2022highresolutionimagesynthesislatent} along with the CFM training. The VAE encoder has $3$ downsampling layers, bringing down the spatial dimension from $[64,64]$ to $[8,8]$ in the latent space. The channel dimension was successively increased from $2$ to $128$ using the sequence $[16, 32, 64, 128]$ and then reduced to $8$ in the final encoder layer. The final latent encodings are flattened and added to the vector field, similar to the CIFAR10 training. We fix the KL regularization parameter $\beta=0.001$.

\paragraph{ImageNet} We follow the same SiT model architecture from \cite{ma2024sitexploringflowdiffusionbased} for our experiments in ImageNet. SiT is a latent flow matching model. For the \texttt{Latent-CFM}, we pretrained a VAE following \cite{rombach2022highresolutionimagesynthesislatent} as described above on the latents of the training data for 200K steps. For reproducibility, we fixed all hyperparameters of the CFM training the same as \cite{ma2024sitexploringflowdiffusionbased}.  

Additional hyperparameter details are presented in Table~\ref{tab:hyperparams}. In addition, following \cite{tong2024improving}, we set the variance of the simulated Gaussian probability path $\sigma_t = 0.01$ for MNIST and $0$ for CIFAR10 and Darcy Flow dataset.
% \AS{Add the additional hyperparameters as a table.}
\begin{table*}[t]
\center
\begin{tabular}{|c|c|c|c|}
\hline
\textbf{Hyperparameters} & \textbf{MNIST} & \textbf{CIFAR-10} & \textbf{Darcy Flow} \\
\hline
Train set size & 60,000 & 50,000 & 10,000 \\

\# steps & 100K & 600K & 100K \\

Training batch size & 128 & 128 & 128  \\

Optimizer & Adam & Adam & Adam  \\

Learning rate & 2e-4 & 2e-4 & 2e-4 \\

Latent dimension & 20 & 256 & 256 \\

number of model channels & 32 & 128 & 128 \\

number of residual blocks & 2 & 2 & 2 \\

channel multiplier & [1,2,2] & [1, 2, 2, 2] & [1, 2, 2, 2] \\

number of attention heads & 1 & 4 & 4 \\

dropout & 0 & 0.1 & 0.1 \\
\hline
\end{tabular}
\vspace{0.2cm}
\caption{Hyperparameter settings used for the experiments on benchmark datasets.}
\label{tab:hyperparams}
\end{table*}

\section{Metrics}
\label{sec:metrics}

\subsection{Wasserstein metric}
Given two batches of samples $(x,y)$, the Wasserstein distance was calculated using the \texttt{SamplesLoss} function with the sinkhorn algorithm from the \texttt{geomloss}\footnote{https://www.kernel-operations.io/geomloss/} Python library.

\subsection{Residual metric on Darcy Flow data}
After rescaling to the original space, we compute the spatially averaged squared residuals of the governing equation across the domain to evaluate the sample 
quality in the 2D Darcy flow experiment. In particular, for each generated sample, we compute
\begin{equation}\label{eqn:darcy_residual}
\begin{aligned}
    R(x) &= \frac{1}{N^2}\Vert f_s(x) + \nabla \cdot [K(x) \nabla p(x)]\Vert_2^2\\
      &= \frac{1}{N^2}\Vert f_s(x) + K(x)\frac{\partial^2p(x)}{\partial x_1^2} + \frac{\partial K(x)}{\partial x_1}\frac{\partial p(x)}{\partial x_1}
      + K(x)\frac{\partial^2p(x)}{\partial x_2^2} + \frac{\partial K(x)}{\partial x_2}\frac{\partial p(x)}{\partial x_2}\Vert_2^2,
\end{aligned}
\end{equation}

where \(N\) is the number of discretization locations in each spatial dimension, \(x_1\) and \(x_2\) represent the spatial coordinates of the domain, and
the partial derivatives in (\ref{eqn:darcy_residual}) are approximated using
central finite differences. 
\begin{figure*}
\centering
\begin{subfigure}{0.4\textwidth} 
    \centering
    \includegraphics[width=\textwidth]{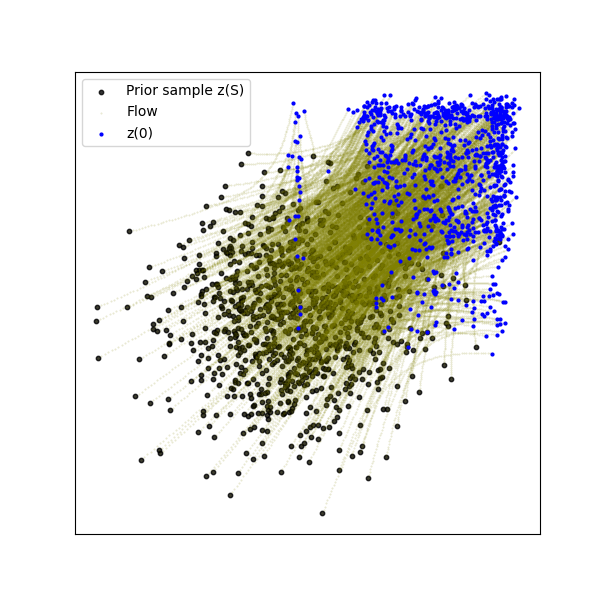}
    \caption{\texttt{Latent-CFM}}
    \label{fig:vrfm_comp_latCFM}
\end{subfigure}
\begin{subfigure}{0.4\textwidth}
    \centering
\includegraphics[width=\textwidth]{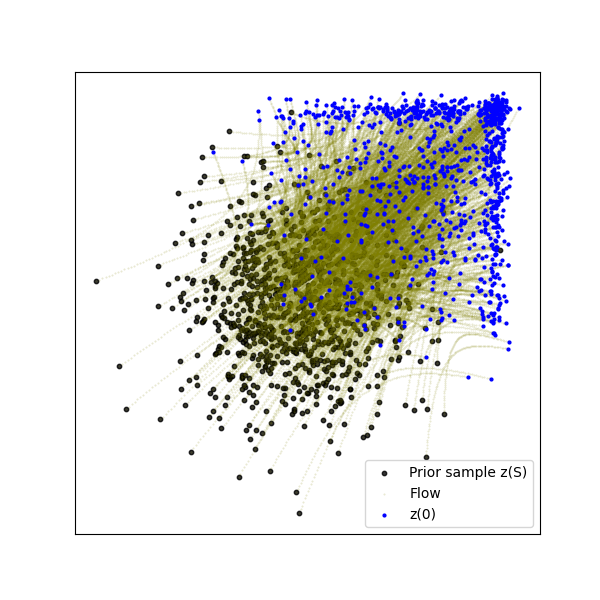}
    \caption{VRFM}
    \label{fig:vrfm_comp_VRFM}
\end{subfigure}
% \hspace{-0.3cm}
% \hspace{-0.3cm}
    \caption{Plot shows the generated trajectories from the source to target distribution for (a) \texttt{Latent-CFM}, and (b) VRFM model. By changing the input to the encoder, \texttt{Latent-CFM} generates samples with an improved multimodal structure similar to the data than the VRFM.}
    \label{fig:VRFM_comparison}
\end{figure*}

\section{Selection of $\beta$}
\label{sec:sel_beta}
\begin{wrapfigure}{r}{0.5\linewidth}
    \centering
    \includegraphics[width=\linewidth]{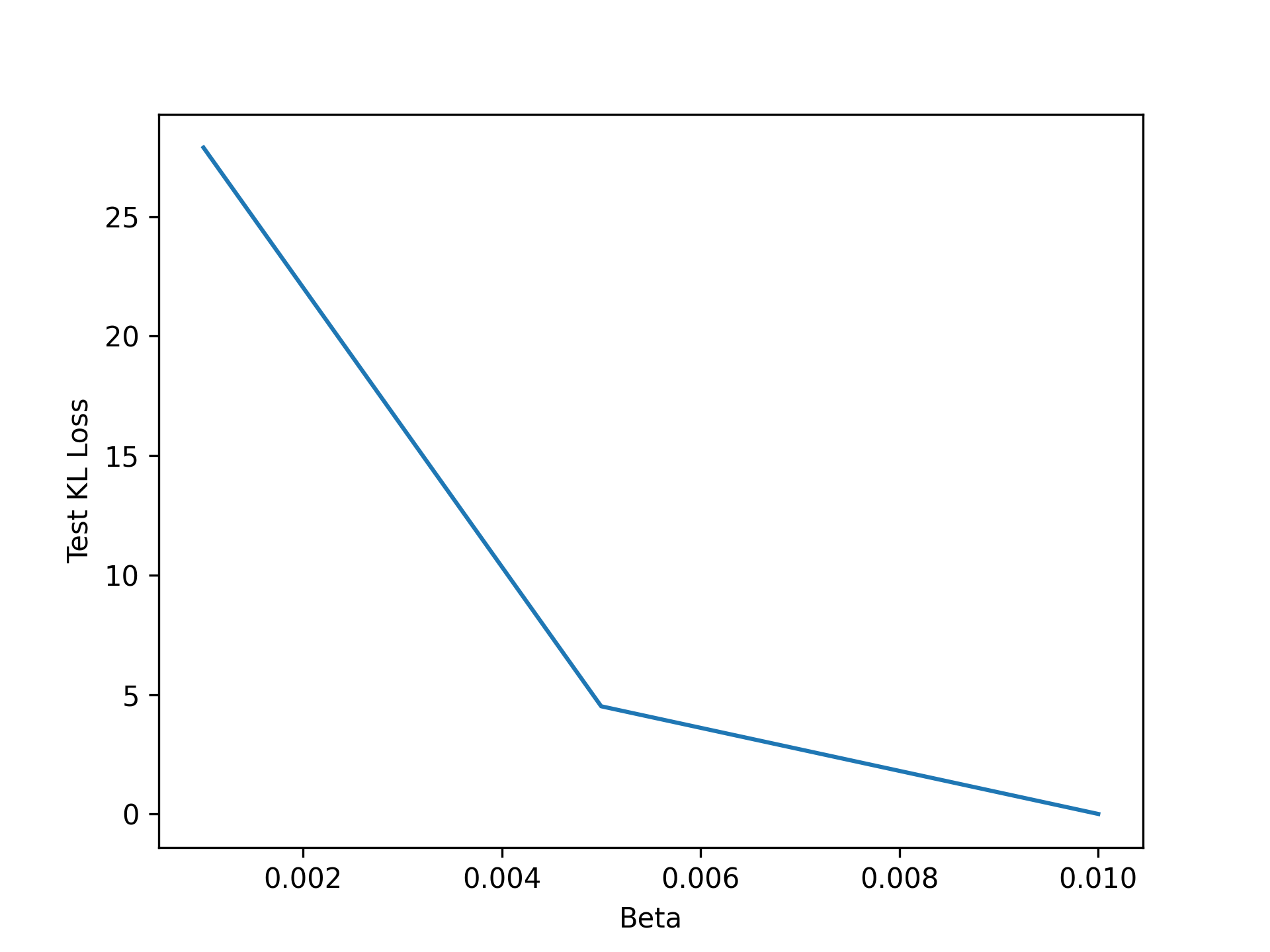}
    \caption{Plot of the regularization parameter $\beta$ and the KL divergence on test set shows an optimal $\beta$ around $5e-3$ on MNIST.}
    % \AS{Add more steps}}
    \label{fig:beta_testKL}
\end{wrapfigure}
An important hyperparameter of \texttt{Latent-CFM} training is the KL regularization parameter $\beta$. It controls the information compression of the latent space \cite{alemi2019deepvariationalinformationbottleneck}. In \texttt{Latent-CFM}, too high $\beta$ results in the posterior collapsing to the prior (KL term $=0$), and too low a value results in the model only learning to reconstruct the training data. In Fig.~\ref{fig:beta_testKL}, we plot the KL term evaluated on the MNIST test dataset vs three $\beta$ values $[0.001,0.005,0.010]$.
We observe an optimal region around $0.005$. On CIFAR10 and Darcy Flow, we observe a good tradeoff between the generation and reconstruction for $\beta=0.001$. We leave a formal strategy for selecting good $\beta$ for \texttt{Latent-CFM} as an interesting future research direction.

\section{Comparison with VRFM}
\label{sec:comp_vrfm}
\texttt{Latent-CFM} loss function is similar to the recently proposed variational rectified flow matching (VRFM) \cite{guo2025variational}. We like to highlight a subtle but key difference between the two methods. VRFM training requires the encoder model to learn the latent $z_t$ from a time-varying input tuple $(x_0,x_1,x_t)$. This is required since the modeling assumption in VRFM is that the vector field $u_t$ is a mixture model for all $t$. \texttt{Latent-CFM} simplifies the model training by modeling the data $x_1$ as a mixture model Eq.7 in the main paper. The encoder only requires the data $x_1$ as input to extract the latent features.
\begin{table}[t]
\centering 
\begin{tabular}{|c|c|}
   \hline
    \textbf{Method} & \textbf{CIFAR10 FID} ($\downarrow$) \\
	\hline
    % VRFM-1 \cite{guo2025variational} &37.2 M & 4.349& 3.582 & 3.561& - & - & - & - \\
    % VRFM-2 \cite{guo2025variational} &37.2 M& 4.484 & 3.614 & \textbf{3.478} & - & - & - & - \\
    VRFM-1 \cite{guo2025variational} & 3.561 \\
    VRFM-2 \cite{guo2025variational} & \textbf{3.478} \\
    \hline
    VRFM $(x_0,x_1,x_t,t)$ & 16.276 \\
    VRFM $(x_1,t)$ & 6.480 \\
    \texttt{Latent-CFM} (pretrained) & \textbf{3.514} \\
   \hline
   \end{tabular}
   \caption{Comparison between VRFM and \texttt{Latent-CFM} in terms of FID on CIFAR10. We were unable to reproduce the FID numbers from \cite{guo2025variational} presented in the top two rows. \texttt{Latent-CFM} shows a similar FID to the best VRFM model from \cite{guo2025variational}. Additionally, we observe performance improvements as we simplify the input to VRFM in our implementation.}
    \label{tab:vrfm_cifar}
\end{table}

Fig.~\ref{fig:VRFM_comparison} shows the generated trajectories from the source to target distribution for (a) \texttt{Latent-CFM}, and (b) VRFM model trained on the 2d triangle dataset. For ease of comparison, we train the VAE encoder in \texttt{Latent-CFM} along with the CFM training, similar to VRFM. The only difference between the two models lies in the input to the encoder, which for VRFM is $(x_0,x_1,x_t,t)$, and for \texttt{Latent-CFM} is $x_1$. We observe that this change helps \texttt{Latent-CFM} to learn the multimodal structure of the data better than the VRFM. This is perhaps due to the violation of the VRM modeling assumption that the vector field random variable $u_t$ is multimodal for all $t$, which is approximately true when $t$ is close to $1$.

\section{Generalization of \texttt{Latent-CFM}}
\label{sec:generalization}
\begin{table*}
    \centering
    \resizebox{0.9\textwidth}{!}{
    \begin{tabular}{|c|c|c|c|c|c|}
    \hline
        % &\multicolumn{2}{c|}{\textbf{I-CFM}} & \multicolumn{3}{c|}{\texttt{DL-CFM}}  \\
        % \hline
      \textbf{Methods} & \textbf{\# Params.}  & \textbf{Sinkhorn} ($\uparrow$) & \textbf{Energy} ($\uparrow$) & \textbf{Gaussian} ($\uparrow$) & \textbf{Laplacian} ($\uparrow$) \\
      \hline
     ICFM  & 35.8M & 675.613 & 12.690 & \textbf{0.0020} & \textbf{0.0018}\\
     \hline
     \texttt{Latent-CFM}  & 36.1M & \textbf{680.500} & \textbf{12.698} & \textbf{0.0020} & \textbf{0.0018}\\
     \hline
    \end{tabular}
    }
    \caption{Table shows the generalization for ICFM and \texttt{Latent-CFM} in terms of different distances (average over 30 batches) from the train dataset on CIFAR10. In terms of most metrics, the two approaches show similar distance from the train dataset.
    }
    \label{tab:generalization}
\end{table*}
\texttt{Latent-CFM} sampling uses latent features $f$ learned from training data samples $q(.|x^{train})$. 
This is uncommon for flow-based modeling and most other deep generative model frameworks, since the model has an additional, non-parametric component that models random features of the data.
Note that there are many non-parametric models that directly use the training set during sampling, e.g. kernel density estimation and Gaussian processes \cite{bishop2006pattern}, while neural networks are sometimes seen as non-parametric, due to their overparametrization and their memorization capabilities, manifested by compressing/memorizing the training data into their weights \cite{zhang2016understanding, arora2018stronger}.
In practice, deep generative models may also memorize features (or even entire samples in the overfitting regime) \cite{gu2023memorization, somepalli2023understanding}.
To assess that our model does not cheat, i.e. copy too much information from the training sample, we perform additional comparisons between generated samples, their conditioning training sample, and the train/test sets at large.

To investigate the generalization of the generated samples beyond the training data, Fig.~\ref{fig:NN} compares the 10 nearest neighbors of 10 \texttt{Latent-CFM} generated samples from the train and test data, and the training images used to generate the samples on CIFAR10. In most cases, we observe that the generated samples, although they share features with, significantly differ from nearest neighbors in the train and test data. This demonstrates that the generated samples don't reconstruct the training data. In addition, for most samples, the training data point used for feature extraction does not appear (except for the red box) within the 10 nearest neighbor samples in the train dataset, demonstrating generalization beyond the training set. 

To quantify generalization beyond training data, Table~\ref{tab:generalization} shows the distance metrics (averaged over 30 batches of size 500) between generated samples using ICFM and \texttt{Latent-CFM} and the training dataset on CIFAR10. We have used the \texttt{geomloss} Python package to calculate the distances. In terms of most distances, both ICFM and \texttt{Latent-CFM} show similar generalization. Additionally, Table~\ref{tab:test_fid} shows FID for ICFM and \texttt{Latent-CFM} for three training checkpoints calculated on the CIFAR10 test dataset. We observed that our approach shows better FID than ICFM across all checkpoints.  

\begin{figure*}
    \centering
    \includegraphics[width=0.9\linewidth]{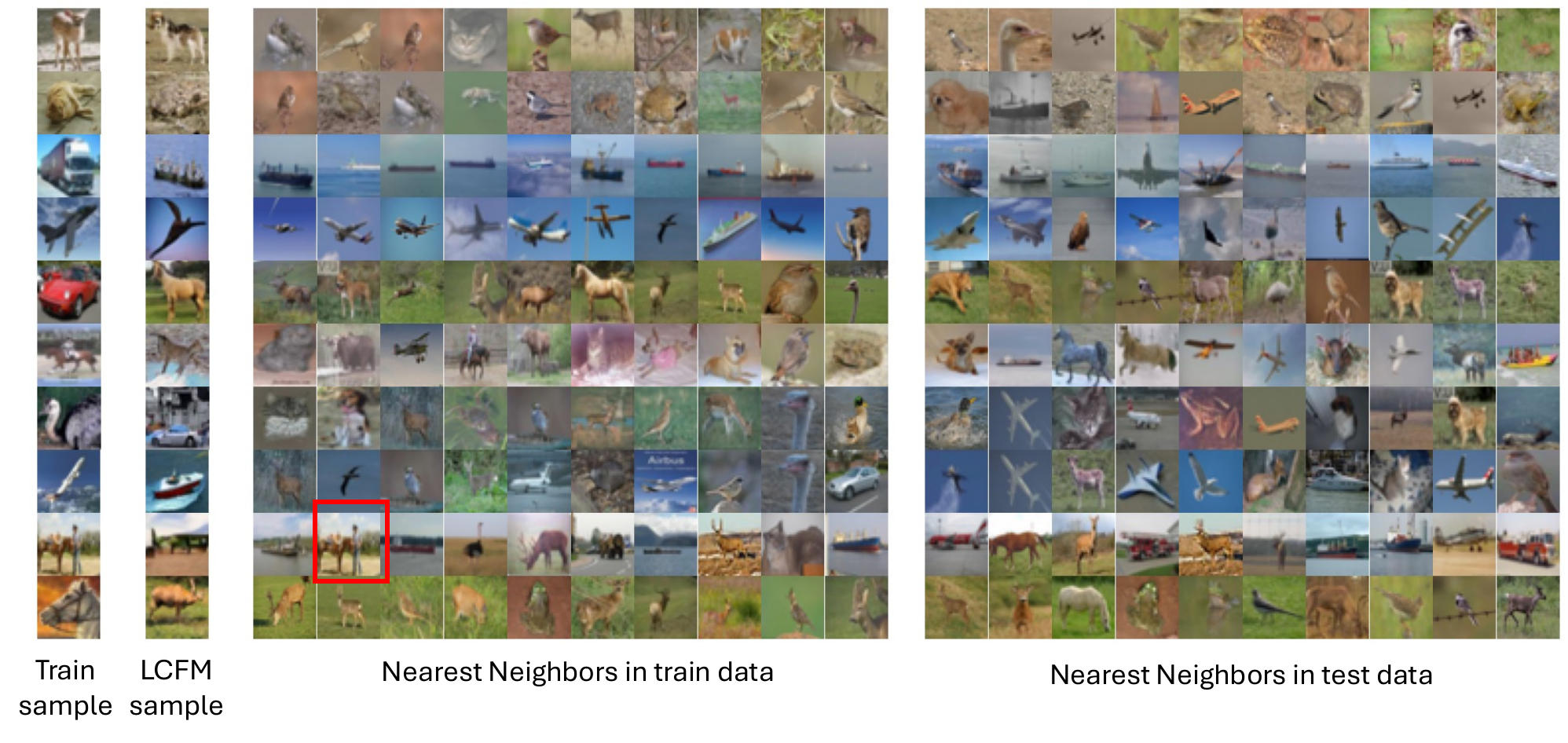}
    \caption{Plot shows 10 nearest neighbors of \texttt{Latent-CFM} samples on the train and test dataset with the training images used for feature extraction (on left). \texttt{Latent-CFM} generated samples generalize well, although they share features with the nearest neighbors in the train data. For one sample, one of the nearest neighbors (marked red) matches the training samples used for feature extraction.}
    \label{fig:NN}
\end{figure*}
\begin{figure*}
    \centering
    \begin{subfigure}{0.48\textwidth}
        \centering
        \includegraphics[width=\linewidth]{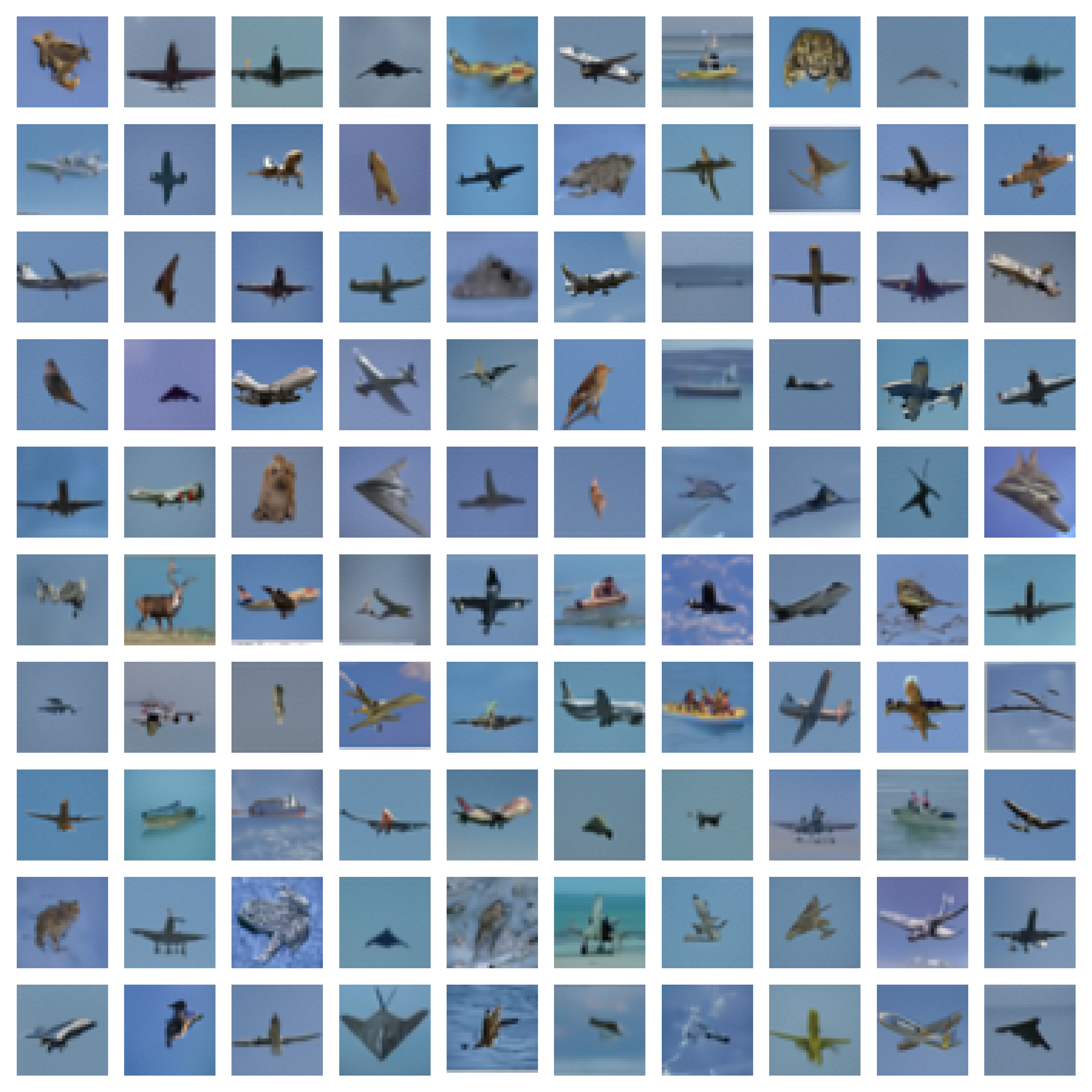}
        \caption{$f_1 \sim q(.|x^{truck}), f_1 \sim q(.|x^{airplane})$}
    \end{subfigure}
    \begin{subfigure}{0.48\textwidth}
        \centering
        \includegraphics[width=\linewidth]{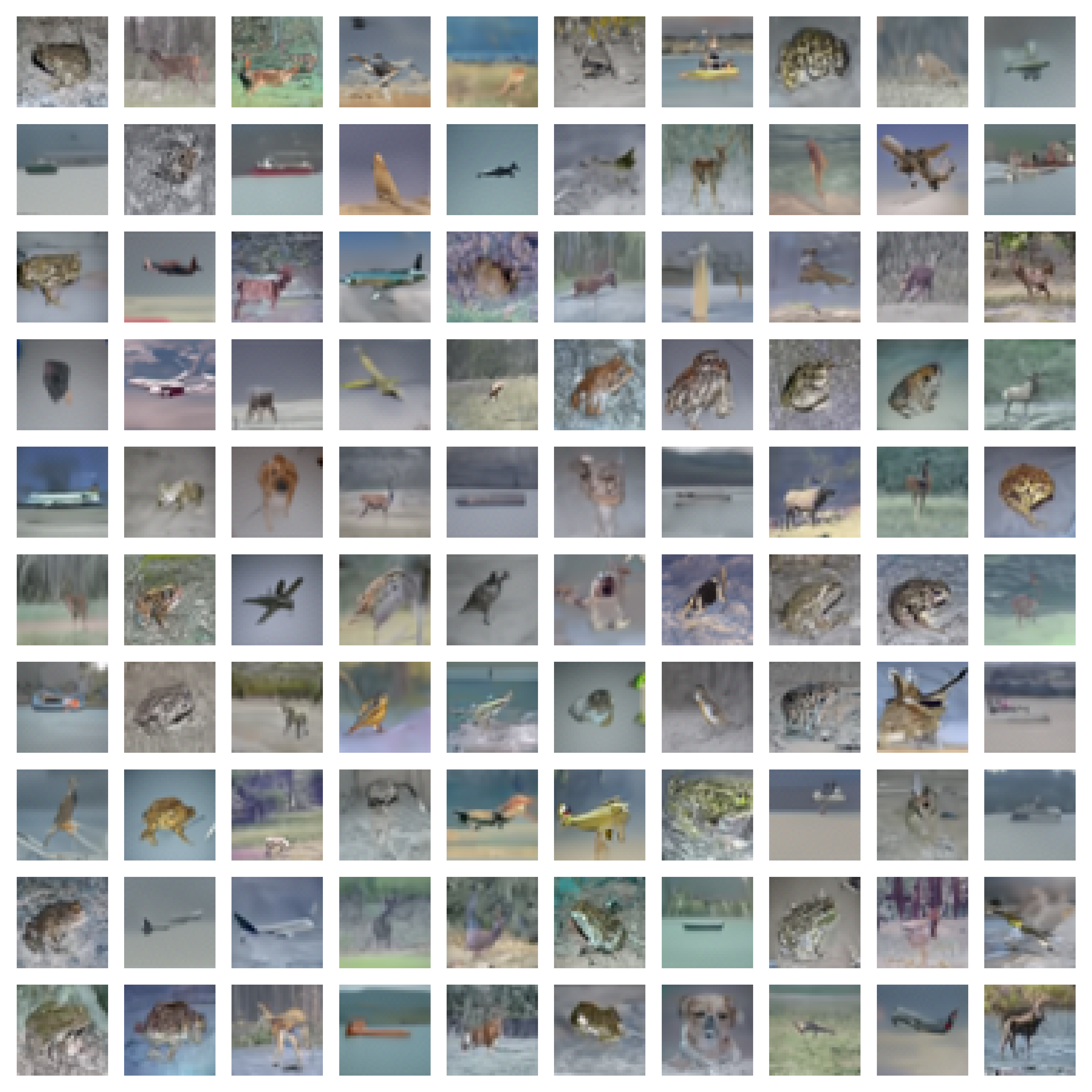}
        \caption{$f_1 \sim q(.|x^{deer}), f_1 \sim q(.|x^{airplane})$}
    \end{subfigure}
    \caption{Expanded set of 100 samples from the two product distributions.}
    \label{fig:composition_100}
\end{figure*}

\section{Latent space traversal on MNIST}
\begin{figure}[t]
    \centering
    \includegraphics[width=\linewidth]{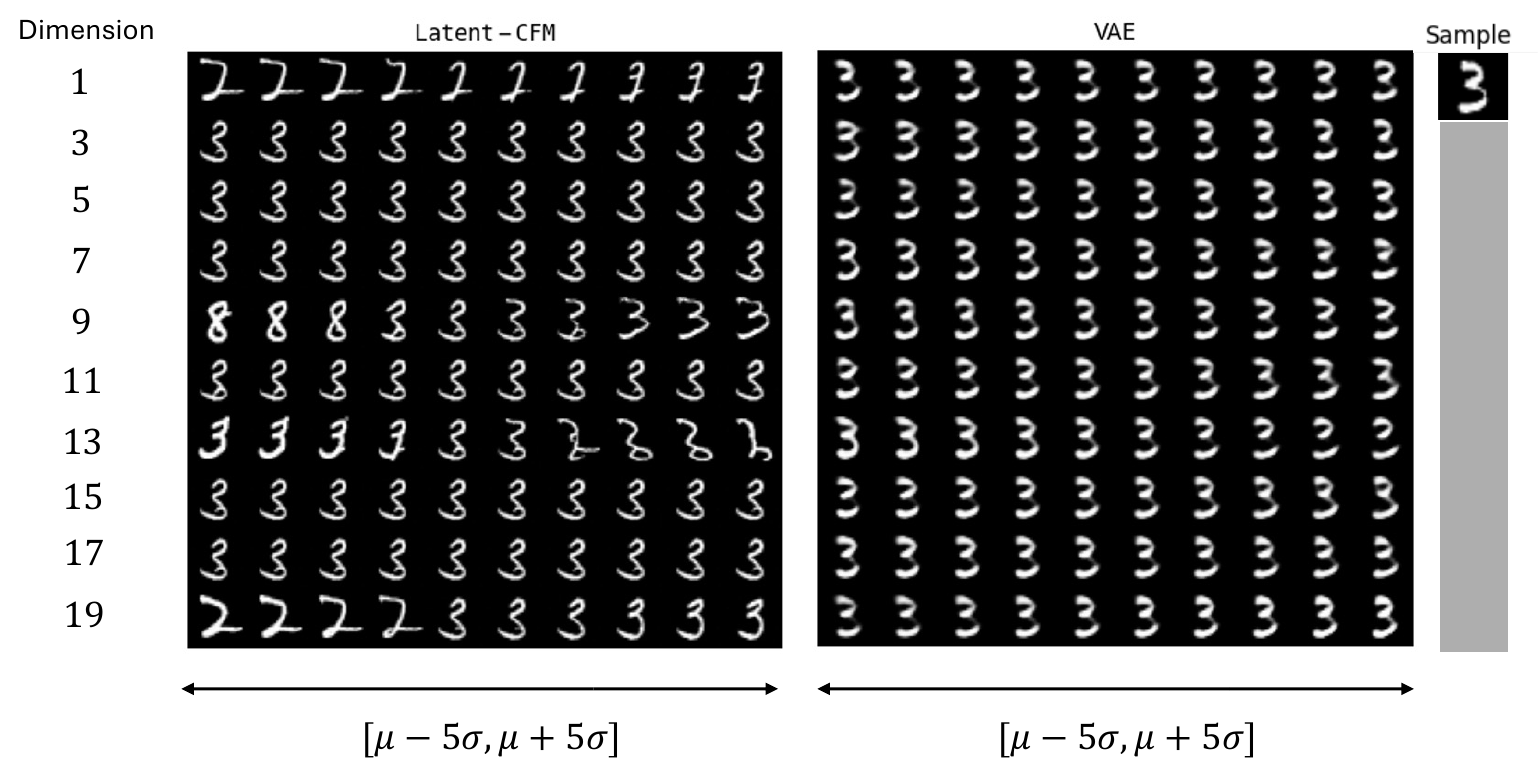}
    \caption{Traversal of the latent space $f$ shows that \texttt{Latent-CFM} generates a diverse set of samples than the pretrained VAE. All generated images for \texttt{Latent-CFM} share the same source samples $x_0 \sim N(0,I)$. The latent space traversal shows we can generate different digits with similar latent structures.}
    \label{fig:VAE_latent}
\end{figure}

\begin{table}[t]
\centering  
\begin{tabular}{|c|c|c|c|}
   \hline
   \multirow{2}{1.5cm}{\textbf{Methods}}  & \multicolumn{3}{c|}{\textbf{CIFAR10 testset FID} ($\downarrow$)} \\
   \cline{2-4} &\textbf{300K} & \textbf{400K} & \textbf{600K} \\
	\hline
	ICFM & 5.769 & 5.742  & 5.652 \\
        \texttt{Latent-CFM} (pretrained) & \textbf{5.710} & \textbf{5.563} & \textbf{5.631} \\
   \hline
   \end{tabular}
   \caption{\texttt{Latent-CFM} compared to I-CFM in terms of testset FID on CIFAR10 for three training checkpoints. Our method exhibits improved FID over ICFM for all checkpoints.}
    \label{tab:test_fid}
\end{table}

% \begin{table}[t]
% \centering 
% \begin{tabular}{|c|c|c|c|}
%    \hline
%     \multirow{2}{1.5cm}{\textbf{Method}} & \multicolumn{3}{c|}{\textbf{CIFAR10 testset FID} ($\downarrow$)} \\
% 	\hline
%     & & 300K & 400K & 600K \\
%     ICFM & 16.276 \\
%     \texttt{Latent-CFM} (pretrained) & \textbf{3.514} \\
%    \hline
%    \end{tabular}
%    \caption{Comparison between VRFM and \texttt{Latent-CFM} in terms of FID on CIFAR10. We were unable to reproduce the FID numbers from \cite{guo2025variational} presented in the top two rows. \texttt{Latent-CFM} shows a similar FID to the best VRFM model from \cite{guo2025variational}. Additionally, we observe performance improvements as we simplify the input to VRFM in our implementation.}
%     \label{tab:vrfm_cifar}
% \end{table}

Fig.~\ref{fig:VAE_latent} shows the latent space traversal for the pretrained VAE model and the \texttt{Latent-CFM} model, which augments the feature learned by the VAE encoder according to Alg.~\ref{alg:vae-cond}. We obtain the latent variables from data samples and generate new samples by perturbing the odd coordinates ($1,3,...,19$) of the $20$-dimensional latent space within the range $[\mu-5\sigma,\mu+5\sigma]$, where $(\mu,\sigma)$ represent the encoded mean and variance.
%We input the data sample (on the right) to both methods, passing it through their encoder-decoder architecture to generate the samples and perturb the odd coordinates ($1,3,...,19$) of the $20$-dimensional latent space within the range $[\mu-5\sigma,\mu+5\sigma]$ where $(\mu,\sigma)$ represents the encoded mean and variance.
For each row, the new samples correspond to perturbing only one coordinate of the data sample latent.
%we fix all other coordinates of the latent variable as per the samples drawn conditioned on the data sample.
For all generated images with \texttt{Latent-CFM}, we fix the samples $x_0$ from the source distribution. We observe that \texttt{Latent-CFM} generates images with significantly more variability and better quality than the baseline VAE model. \texttt{Latent-CFM} has generated the same digit (the digit 3) but with different styles. In addition, we observe that perturbing certain latent coordinates (for example, coordinate 9,19) generates different digits with similar structure.

\section{Composition of feature-conditioned distributions}
\label{sec:compose_details}
% This section lightly explores the compositionality of LCFM and present initial insights
% in the composed distribution. 
This section describes the algorithm and derivations for our sampling algorithm from the product of two feature-conditioned \texttt{Latent-CFM} models. The key ingredient involves the relation between the vector field $u_t(.)$ and the score $\nabla \log p_t(.)$ where $p_t(.)$ is the probability path at time $t$. The following Lemma from \cite{zheng2023guidedflowsgenerativemodeling} describes this relation for the Gaussian probability paths.

\begin{lemma}\label{lemma1}
    Let $p_t(x|x_1) = N(x|\alpha_tx_1,\sigma_t^2I)$ be a Gaussian Path defined by a scheduler $(\alpha_t,\sigma_t)$, then its generating vector field $u_t(x|x_1)$ is related to the score function $\nabla\log p_t(x|x_1)$ by,
    \begin{equation}\label{eq:score_vector}
        u_t(x|x_1) = a_tx + b_t\nabla \log p_t(x|x_1)
    \end{equation}
    where,
    \begin{equation}\label{eq:score_vector_params}
        a_t=\frac{\dot{\alpha_t}}{\alpha_t}, b_t=(\dot{\alpha_t}\sigma_t - \alpha_t\dot{\sigma_t})\frac{\sigma_t}{\alpha_t}
    \end{equation}
\end{lemma}

In this paper, we used linear interpolation paths in \texttt{Latent-CFM} $x_t = tx_1 + (1-t)x_0$, where $x_0 \sim N(0,I)$. It is easy to show that, $a_t = \frac{1}{t}, b_t = \frac{1-t}{t}$. Therefore, we can use Eq.~\ref{eq:score_vector} to convert the learned vector field $v_{\hat{\theta}}(.;.,t)$ to score estimator $s_{\hat{\theta}}(.;.,t)$.

We aim to draw samples from the product probability path $p^1_t = \prod_{i=1,2} p_t(x|f_i)$ where $f_i \sim q(.|x_i)$ and $(x_1, x_2)$ are two images. Using Lemma~\ref{lemma1}, we can derive the vector field underlying the product, $u^1_t(x) = -a_t x + u_t(x|f_1) + u_t(x|f_2)$. During the reverse ODE sampling, we replace the true vector fields with their learned networks, 
\begin{equation}\label{eq:compose_vec}
v^1(x;.,t) = -a_t x + v_{\hat{\theta}}(x;f_1,t) + v_{\hat{\theta}}(x;f_2,t)    
\end{equation}
With these derivations, we perform the predictor-corrector sampling using Langevin dynamics from \cite{song2020score} to sample from the product distribution. Algorithm~\ref{alg:compose_sampling} describes the steps of the sampling process.

\begin{algorithm}[t]
\caption{Sampling from product of feature-conditioned models}\label{alg:compose_sampling}
\begin{algorithmic}[1]
\STATE trained $v_{\hat{\theta}}(.,f_i,.)$ where $f_i\sim q_{\hat{\lambda}}(.|x_i)$, number of ODE steps $n_{ode}$, number of Langevin steps $n_\ell$, drift and diffusion schedulers $(\epsilon_t^{drift}, \epsilon_t^{diffusion})_{t\in[0,1]}$ 
% \STATE $\hat{\Q} \leftarrow$ \texttt{REMEDI}$(\Q)$.
% \STATE Select $K$ training samples $(x_1^{train},...,x_K^{train})$
% \STATE Sample latent variables $f_i \sim q_{\hat{\lambda}}(f|x_i^{train})$ for all $i=1,...,K$
\STATE Sample $x_0 \sim N(0, I)$
\STATE $h \leftarrow \frac{1}{n_{ode}}$ 
\FOR{$t=0,h,...,1-h$}
\STATE $x_{t+h} =$ ODEstep($v^1(x_t;.,t),x_t$) \RightComment{\textbf{Predictor step}}
\FOR{$j=1,...,n_\ell$}
\STATE $z \sim N(0,I)$
\STATE $x_{t+h} = x_{t+h} + \epsilon_{t+h}^{drift} s_{\hat{\theta}}(x_{t+h};.,t) + \sqrt{2\epsilon_{t+h}^{diffusion}}z$ \RightComment{\textbf{Corrector step}}
\ENDFOR
\ENDFOR
\STATE \textbf{return} $x_1$
\end{algorithmic}
\end{algorithm}

As stated in Section \ref{sec:latent_analysis}, assuming conditionally independent latent variables \(f_1\) and \(f_2\) given \(x\), we can construct a new distribution. Repeatedly using Bayes theorem, we can obtain  
\begin{equation*}
   p(x|f_1, f_2) = \frac{p(x|f_1)p(x|f_2)}{p(x)}\cdot\frac{p(f_1)p(f_2)}{p(f_1, f_2)} \propto  \frac{p(x|f_1)p(x|f_2)}{p(x)}.
\end{equation*}
Such construction requires a marginal data distribution \(p(x)\) where
\(p(x|f_1, f_2)\) is a proper conditional. This means such \(p(x)\) allows
\(f_1\) and \(f_2\) to exist simultaneously for some \(x = x^\prime\). However,
without a clear semantic meaning or disentanglement of the latent space, \(p(x|f_1, f_2)\)
is not well-defined. This might explain the unresolved center objects in the
samples from the composed distribution.

To perform an effective composition in this case,
\cite{bradley2025mechanismsprojectivecompositiondiffusion} shows that
\(p(x|f_1)\) and \(p(x|f_2)\) need to be independent. Therefore, it asks the
latent variables to represent orthogonal concepts in \(x\), which prompts the 
future direction of learning a latent space encapsulating mutually independent 
features of the input data for better compositional generation.

\section{Additional results on Darcy Flow data}
\label{sec:additional_darcy}
\begin{figure}[t]
    \centering
    \includegraphics[width=\linewidth]{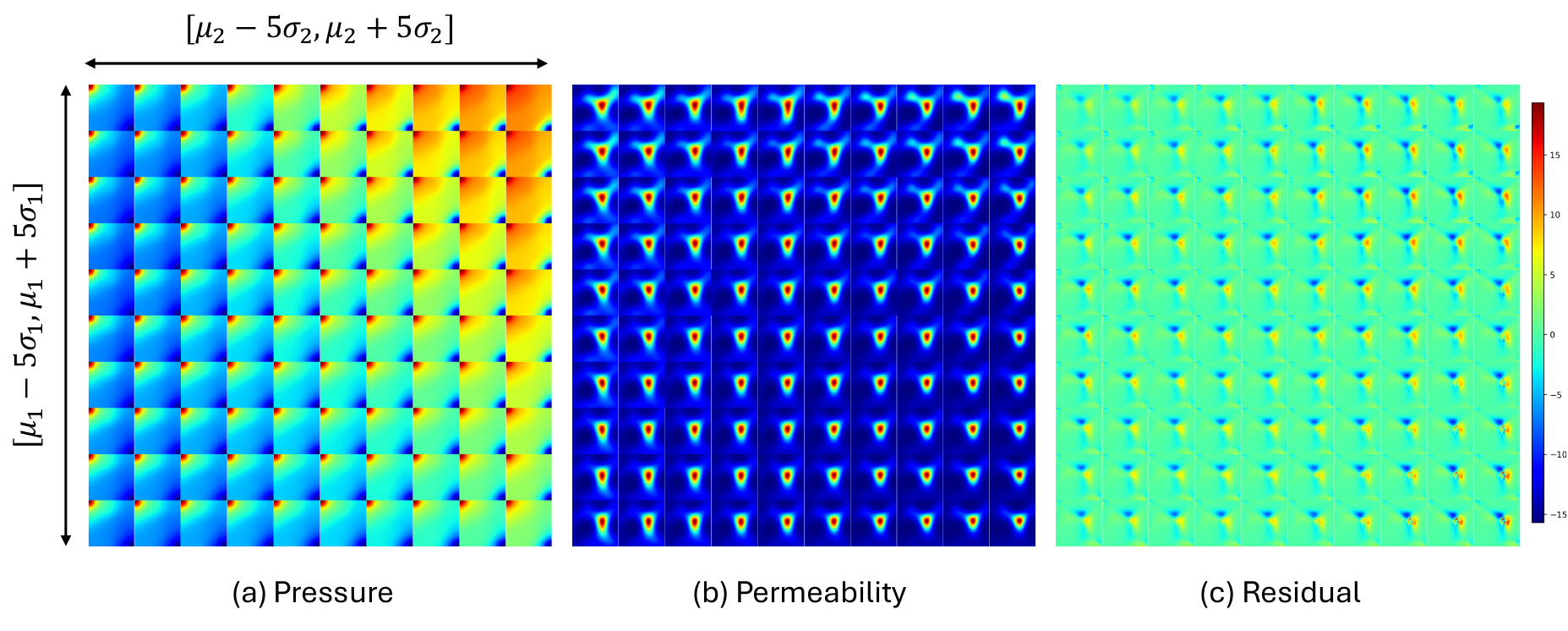}
    \caption{Plot shows the generated pressure and permeability fields and the residual measuring physical accuracy of generation from varying the two latent dimensions within $[\mu-5\sigma, \mu+5\sigma]$. We observe that traversing the latent space produces a variety of samples for both fields. In addition, the variation in the latent space seems to maintain the physical consistency of the samples in terms of the residual. All images share the same source samples $x_0$.}
    \label{fig:lat_traversal_darcy}
\end{figure}

% \section{2d Latent space traversal}
% To study the learning of latent space, we trained a \texttt{Latent-CFM} model with a 2d latent space. We use the same pretrained VAE encoder ($\sim1$M parameters) as the previous section, and fix the output dimension of the final trainable linear layer to produce mean and log-variance of a 2d Gaussian distribution.

In Fig.~\ref{fig:lat_traversal_darcy}, we plot generated fields and the residual metric by varying the two dimensions of the latent variable within the range $[\mu_i-5\sigma_i, \mu_i+5\sigma_i], i =1,2$, where $i$ denotes the dimension, and $(\mu_i, \sigma_i)$ denotes the mean and standard deviation respectively. We fix the same sample $x_0 \sim N(0,I)$ from the source distribution for all generations. We observe that traversing the latent space manifold has produced a variety of generated fields for both variables. In addition, \texttt{Latent-CFM} with varying latent variables seems to have generated physically consistent samples in terms of the residual. For the 2d grid of latent variables in Fig.~\ref{fig:lat_traversal_darcy}, the mean and median residual MSE were $3.614$ and $3.467$, respectively. 
% In addition, with this model, we observed a median residual RMSE of $3.18$ across $500$ generated samples, which is the lowest across methods in Table 3 of the main paper. 
This demonstrates that traversing the latent space helps to generate physically accurate samples that share semantic similarities.

\section{\texttt{Latent-CFM} joint training results on ImageNet-256}
\label{sec:jont_imnet}
In this section, we compare the two training strategies of \texttt{Latent-CFM}, (1) with a \textbf{pretrained} VAE encoder with the final layer parameters $\lambda$ updated with the CFM training, and (2) an encoder \textbf{jointly trained} with the learned vector field on the ImageNet-256 dataset. For both models, we use the same architectures for VAE \cite{rombach2022highresolutionimagesynthesislatent}, and the vector field network \cite{tong2024improving}. Other hyperparameters, such as per-GPU batch size and learning rate, follow from \cite{ma2024sitexploringflowdiffusionbased} and are kept the same for both models. 

\texttt{Latent-CFM} with a pretrained encoder significantly improves ($\sim 2.4$ steps/second) the training speed from the jointly trained model ($\sim 1.3$ steps/second). Table~\ref{tab:joint_pretrain} shows the FID over training steps for both models. We observe that \texttt{Latent-CFM} with a pretrained encoder outperforms the jointly trained model across all training steps.

\begin{table}[!htp]
\centering 
\begin{tabular}{|c|c|c|}
   \hline
    \textbf{Steps} & \texttt{Latent-CFM} (pretrained) FID ($\downarrow$) & \texttt{Latent-CFM} (joint) FID ($\downarrow$)\\
	\hline
    200K & \textbf{11.781} & 44.259\\
    400K & \textbf{7.324} & 27.186\\
    600K & \textbf{6.450} & 20.608\\
   \hline
   \end{tabular}
   \caption{Comparison between \texttt{Latent-CFM} with a \textbf{pretrained} and \textbf{jointly trained} encoder in terms of FID on ImageNet-256. \texttt{Latent-CFM} with a pretrained encoder performed significantly better than the jointly trained model across all training steps.}
    \label{tab:joint_pretrain}
\end{table}

\section{Computational cost}

For our experiments, we used between 1, 4, and 8 NVIDIA A100 GPUs to train all models. On MNIST, CIFAR10, and Darcy Flow datasets, it took approximately $3$, $16$, and $3.5$ hours to complete $100K$, $600K$, and $100K$ steps of the \texttt{Latent-CFM} training, respectively. On ImageNet-256, it took $4.5$ days to train \texttt{Latent-CFM} with a pretrained encoder model for $800K$ steps. On Darcy Flow and ImageNet data, we pretrained a VAE model, which took approximately $1.7$ and $5.5$ hours to complete $100K$ and $200K$ training steps, respectively.

\section{Ablation Studies} In this section, we present ablation studies for \texttt{Latent-CFM} on CIFAR10 data with respect to (1) encoder architecture, (2) latent space dimension, and (3) finetuning strategy.

\paragraph{VAE Architectures}
We ran an ablation study comparing the generation quality of $\texttt{Latent-CFM}$ with two pretrained VAE architectures on the CIFAR10 dataset. In the paper, we showed results for $\texttt{Latent-CFM}$ with a frozen $\sim 20$M parameter resnet18 VAE model\footnote{https://github.com/Lightning-Universe/lightning-bolts/}. In addition, we pretrained a $\sim 26$M CNN VAE model following the Stable-Diffusion repository \cite{rombach2022highresolutionimagesynthesislatent} for 100K steps. We fixed all other $\texttt{Latent-CFM}$ parameters except the frozen VAE encoder. Table~\ref{tab:fid_results} shows the FID from both models for three training checkpoints on CIFAR10. We observe that both models perform similarly across the training checkpoints, with LCFM-CNN achieving slightly better FID after 600K steps.

\begin{table}[htbp]
\centering
\begin{tabular}{lcc}
\toprule
Training steps & FID LCFM-resnet18 ($\downarrow$) & FID LCFM-CNN ($\downarrow$) \\
\midrule
300K & \textbf{3.566} & 3.572 \\
400K & \textbf{3.457} & 3.481 \\
600K & 3.514 & \textbf{3.506} \\
\bottomrule
\end{tabular}
\caption{FID comparison across training steps}
\label{tab:fid_results}
\end{table}

\begin{table}[htbp]
\centering
\begin{tabular}{lc}
\toprule
Latent dimension & FID ($\downarrow$) \\
\midrule
16  & 3.916 \\
32  & 3.747 \\
64  & \textbf{3.514} \\
128 & 3.584 \\
256 & \textbf{3.517} \\
512 & 3.556 \\
\bottomrule
\end{tabular}
\caption{FID across latent dimensions}
\label{tab:latent_fid}
\end{table}

\paragraph{Latent dimension }
We ran an ablation study for $\texttt{Latent-CFM}$ w.r.t the latent dimension on the CIFAR10 dataset. Table~\ref{tab:latent_fid} shows FIDs from varying the dimension of the trainable stochastic layer in $\texttt{Latent-CFM}$ within the range 16-512. We observe that FID increases for small latent spaces ($\leq32$). The larger latent spaces tend to have similar FIDs ($\sim 3.5$).  

\paragraph{VAE finetuning with an adapter}
In the paper, to train $\texttt{Latent-CFM}$ model with a pretrained VAE encoder, we freeze the encoder parameters and add a trainable MLP head to map to the mean and variance of the latent space. However, we can choose a different strategy where we add a low-rank adapter (LoRA) \cite{hu2021loralowrankadaptationlarge} to the final layer of the VAE encoder and train only the adapter parameters. Table~\ref{tab:fid_mlp_lora} shows the FID resulting from $\texttt{Latent-CFM}$ training with the two finetuning strategies on CIFAR10. Both strategies introduce very few ($<1\%$) parameters to the model. We observe that the trainable MLP head for the frozen VAE encoder produces lower FID across training steps on CIFAR10.

\begin{table}[htbp]
\centering
\begin{tabular}{lcc}
\toprule
Training steps & FID LCFM w MLP finetuning ($\downarrow$) & FID LCFM w LoRA ($\downarrow$) \\
\midrule
300K & \textbf{3.566} & 3.713 \\
400K & \textbf{3.457} & 3.727 \\
600K & \textbf{3.514} & 3.743 \\
\bottomrule
\end{tabular}
\caption{FID comparison with MLP finetuning and LoRA}
\label{tab:fid_mlp_lora}
\end{table}

\begin{figure*}[htp]
    % \centering
    % \includegraphics[width=\linewidth]{img/MNIST_final_samples1.pdf}
\centering
\begin{subfigure}[b]{0.6\textwidth} % Width specified here
    \centering
    \includegraphics[width=\textwidth]{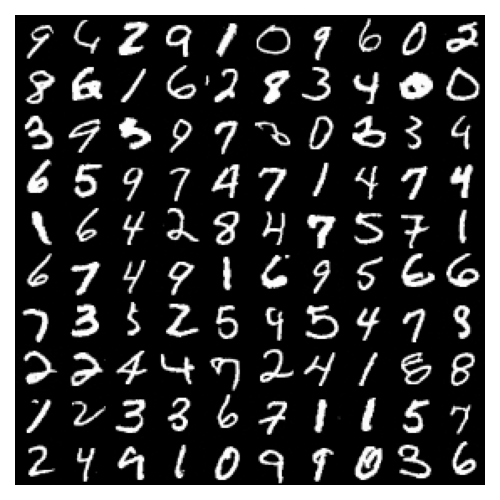}
    \caption{Generated images for MNIST using \texttt{Latent-CFM}}
    \label{fig:mnist_sample}= % Added label for subfigure
\end{subfigure}
\hfill
\begin{subfigure}[b]{0.6\textwidth} % Width specified here
    \centering
    \includegraphics[width=\textwidth]{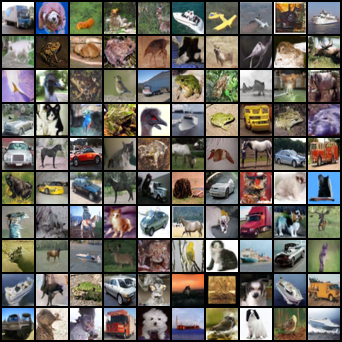}
    \caption{Generated images for CIFAR10 using \texttt{Latent-CFM}}
    \label{fig:cifar_sample} % Label retained
\end{subfigure}
    % \caption{Generated samples on MNIST dataset show improved generation quality of \texttt{Latent-CFM} compared to I-CFM as shown by lower Wasserstein distance.}
    \label{fig:mnist_unconditional}
\end{figure*}

\end{document}